\documentclass{article} %
\usepackage{preprint,times}
\preprint

\usepackage[T1]{fontenc}
\usepackage[utf8]{inputenc}
\usepackage{microtype}
\usepackage{amsmath,amssymb,amsthm,mathtools,bm}
\usepackage{booktabs,array}
\usepackage{enumitem}
\usepackage{minitoc}
\usepackage{todonotes}
\usepackage{xcolor}
\usepackage[table]{xcolor}
\usepackage[most]{tcolorbox}
\usepackage{tikz}
\usetikzlibrary{arrows.meta,positioning,calc,fit}
\usepackage{hyperref}

\usepackage[nameinlink,capitalise,noabbrev]{cleveref}
\usepackage{notation}
\usepackage{wrapfig}
\usepackage{graphicx}
\usepackage{subcaption}
\usepackage{tabularx}

\usepackage{authblk}

\renewcommand{\url}[1]{%
  \href{#1}{\texttt{\detokenize{#1}}}%
}

\definecolor{pcblue}{RGB}{42,92,138}
\definecolor{pcgreen}{RGB}{46,118,93}
\definecolor{pcred}{RGB}{153,66,66}
\definecolor{pcgold}{RGB}{162,117,27}
\definecolor{pcgray}{RGB}{90,96,104}
\definecolor{pcpurple}{RGB}{105,73,148}
\definecolor{pclight}{RGB}{246,248,251}

\hypersetup{
  colorlinks=true,
  linkcolor=pcblue,
  citecolor=pcgreen,
  urlcolor=pcblue
}

\tcbset{
  subtlebox/.style={
    enhanced,
    breakable,
    frame hidden,
    colback=white,
    coltitle=black,
    colbacktitle=white,
    fonttitle=\bfseries,
    left=3mm,
    right=0mm,
    top=1mm,
    bottom=1mm,
    boxsep=0pt,
  }
}

\newtcolorbox{takeawaybox}[1][]{
  subtlebox,
  borderline west={3pt}{0pt}{pcgreen!72},
  title={Main takeaway},
  #1
}

\newtcolorbox{intuitionbox}[1][]{
  subtlebox,
  borderline west={3pt}{0pt}{pcblue!68},
  title={Intuition},
  #1
}

\newtcolorbox{interpretationbox}[1][]{
  subtlebox,
  borderline west={3pt}{0pt}{pcgold!72},
  title={Interpretation},
  #1
}

\newtcolorbox{warningbox}[1][]{
  subtlebox,
  borderline west={3pt}{0pt}{pcred!68},
  title={Important limitation},
  #1
}

\newtcolorbox{notationbox}[1][]{
  subtlebox,
  borderline west={3pt}{0pt}{pcgray!55},
  title={Notation},
  #1
}

\newtcolorbox{clarificationbox}[1][]{
  subtlebox,
  borderline west={3pt}{0pt}{pcpurple!68},
  title={Key clarification},
  #1
}

\title{Counterfactual Predictions in Scientific \\ Emulators Without Controlled Experiments}
\author{
\textbf{Dingling Yao}$^{1}$,
\textbf{Kahaan Gandhi}$^{2}$,
\textbf{Valentin Duruisseaux}$^{3}$,
\textbf{Boris Bonev}$^{3}$
\vspace{-1mm}

\textbf{Francesco Locatello}$^{1,\dagger}$,
\textbf{Anima Anandkumar}$^{2,\dagger}$
\\
\vspace{0em}
\normalfont
$^{1}$Institute of Science and Technology Austria\quad
$^{2}$California Institute of Technology\quad
$^{3}$NVIDIA\quad
\vspace{0.25em}
\small $^\dagger$Equal supervision.
}

\begin{document}
\doparttoc %
\faketableofcontents %
\maketitle

 \vspace{-3mm}
 
\begin{abstract}
\looseness=-1
Many scientific questions require reasoning about what was never observed:
\textit{What if the conditions, interventions, or history had been different?}
Models can predict accurately on observed data yet fail on such what-if queries when correlated inputs are varied independently.
A common remedy is to add controlled simulation data in which these factors are explicitly disentangled, but this requires access to a simulator, can be computationally expensive, and inherits the simulator's modeling assumptions.
We introduce \emph{ReRoute}, a framework for targeted scientific what-if prediction that combines factual data with partial mechanistic knowledge, without requiring controlled intervention data for adaptation.
\emph{ReRoute} fixes the queried input of a pretrained backbone to a reference value, reintroduces its variation through a known mechanistic pathway, and fine-tunes on the original factual data, while leaving downstream effects to the learned dynamics.
We provide a causal identification result for this construction under explicit structural assumptions, with the core argument machine-checked in Lean.
After showing that \emph{ReRoute} achieves highly accurate counterfactual predictions in a controlled advection--diffusion system where exact responses are available, we turn to state-of-the-art climate emulation.
On held-out coupled-climate interventions, \emph{ReRoute} reduces aggregate climate error by 18.2--31.8\% under severe CO$_2$ distribution shifts while preserving skill under standard conditions, at a small fraction of the cost of retraining on additional controlled simulations, without even accounting for the substantial expense of generating such data. Finally, on an emulator trained from historical ERA5 reanalysis, where no counterfactual reference exists, \emph{ReRoute} preserves substantially more of the surface warming implied by the observed boundary conditions under a fixed-CO$_2$ counterfactual.
\end{abstract}

\section{Introduction}

\looseness=-1 Data-driven surrogate models, or emulators, are becoming a standard tool for approximating complex scientific systems. Given the current state of a system together with external conditions, they predict its target at a fraction of the cost of traditional simulation. But scientific questions often requires reasoning beyond predictions: We want to know how a system would respond under a change we have not directly observed: how would the atmosphere respond if CO$_2$ level had never risen \citep{watt2025ace2,clark2026disentangling}, how would a biological system respond to a different treatment or dose \citep{feuerriegel2024causal}, or how would a physical system evolve under a modified forcing \citep{rigas2017weakly}? Answering such \emph{what-if questions} requires models to understand the mechanism of the underlying system, which is one step further than merely matching the observational distribution.

A central challenge arises when the training data contain only correlated rather than independently varied \emph{drivers} (i.e. variables whose values influence the outcome of interest, such as environmental conditions, treatments, or physical forcings). 
A flexible model can use one input as a proxy for another and still achieve excellent predictive accuracy, while learning the wrong underlying mechanism. Standard validation does not expose this shortcut when the same dependencies persist at test time; it becomes visible only when a what-if query breaks them. A direct solution is to augment training with controlled experiments in which the relevant drivers are varied independently. Such experiments provide exactly the missing variation, but require access to a simulator, can be costly to generate at scale, and inherit the simulator's own modeling assumptions and biases. For many real scientific systems, the corresponding controlled perturbation data are simply unavailable.

This challenge is particularly relevant in weather and climate science, where data-driven systems are increasingly prominent, as evidenced by recent advances in weather forecasting such as FourCastNet~\citep{pathak2022fourcastnet,bonev2025fourcastnet} and GenCast~\citep{price2025probabilistic}, and in climate emulation through models including ACE \citep{watt2023ace}, probabilistic extensions based on spherical DYffusion \citep{ruhling2024probablistic}, and the ACE2 series~\citep{watt2025ace2,clark2025ace2}.
Historical climate trajectories illustrate the challenge of co-varying drivers: atmospheric CO$_2$ rises as sea surface temperature (SST) warms and sea ice declines. An emulator trained on these trajectories can therefore reproduce historical warming while relying too strongly on CO$_2$ as a proxy for the atmospheric response \citep{watt2025ace2,zhang2025equilibrium,clark2026disentangling}. If CO$_2$ is held fixed while SST and sea ice are varied, this learned dependence can yield unrealistic responses. The model is thus not necessarily extrapolating beyond the range of any single input, but rather to an \emph{unseen combination of inputs}. The controlled-experiment approach provides a direct remedy in this setting: additional physics-based simulations can be generated in which SST and CO$_2$ are varied more independently. Recent climate work uses precisely this strategy to decorrelate these drivers during training, substantially improving the separation of their effects \citep{clark2026disentangling}.

Scientific models, however, often come with partial mechanistic knowledge. We may know how a targeted driver first enters the system, for example, through a radiative, chemical, mechanical, or physiological pathway, but without knowing enough to model all of its downstream interactions. This suggests a different source of information for answering what-if questions:

\vspace{-0.2mm}
\begin{quote}
\emph{Can partial mechanistic knowledge enable reliable what-if predictions when controlled training experiments are unavailable?}
\end{quote}
\vspace{-0.2mm}

\looseness=-1 \textbf{Proposed approach.} We introduce \emph{ReRoute}, a framework that answers targeted what-if questions by combining observational data with known mechanistic pathways, without requiring controlled or interventional training data.
Let $\sb$ denote the input variable targeted by the what-if query. \emph{ReRoute} fixes the value seen by the pretrained backbone to a reference value $\sb_0$ and reintroduces its factual value through a supplied mechanistic pathway $\Gamma(\sb_0,\sb)$, which describes the first level response of intervening the system from $\sb_0$ to $\sb$ while leaving remaining responses to be learned from data.

We demonstrate that \emph{ReRoute} admits a causal identification guarantee: under explicit structural assumptions, the target interventional response is identifiable from observational trajectories together with the known mechanism. The formal conditions and Lean-verified proofs are provided in \cref{app:proofs}.

\begin{clarificationbox}[title={}]
\looseness=-1 \textbf{When does \emph{ReRoute} apply?} \emph{ReRoute} is a general scientific-ML framework.
It applies when three ingredients are available: (1) a pretrained emulator, (2) a targeted input whose variation is entangled with other inputs in the training data, and (3) partial knowledge coming from physics, chemistry, biology, or other domain sciences which only needs to characterize how this targeted input enters the system, not the full chain of subsequent effects. \emph{ReRoute} uses this pathway to route the targeted perturbation while leaving the remaining dynamics learned from data. 
\end{clarificationbox}

This setting arises across scientific domains when a pretrained emulator takes several variables as inputs, although they were not varied independently in its training data. In historical climate data such as ERA5 \citep{hersbach2020era5}, CO$_2$ forcing and SST co-evolve, so an emulator may entangle their effects. Radiative physics specifies how a CO$_2$ intervention modifies the atmospheric forcing while preserving the SST evolution, leaving the emulator to predict the downstream atmospheric response. In chemical reaction datasets, catalyst concentration is often set as a fixed fraction of reactant concentration, so experiments with more reactant also contain more catalyst. Chemistry specifies whether an increase should preserve the catalyst concentration or the catalyst-to-reactant ratio, while the emulator predicts reaction rate, yield, and selectivity \citep{clayden2012organic,atkins2022physical}. Exercise datasets often associate greater intensity with shorter duration because intense workloads cannot be sustained for as long. Physiology specifies whether an increase in intensity should preserve duration, total work, or energy expenditure, while the emulator predicts cardiovascular and metabolic responses \citep{kenney2025physiology}. Nutrition datasets may associate higher caloric intake with greater protein intake because increasing food intake often changes both. Nutritional knowledge specifies the macronutrient composition of the additional calories, allowing the emulator to predict effects on metabolism and body composition \citep{gropper2022advanced}. In all these examples, domain knowledge need only translate the targeted intervention into a scientifically coherent change to the model inputs, leaving the pretrained emulator to predict the downstream response.

\begin{figure} \vspace{-1.4mm}
  \centering  \includegraphics[width=\linewidth]{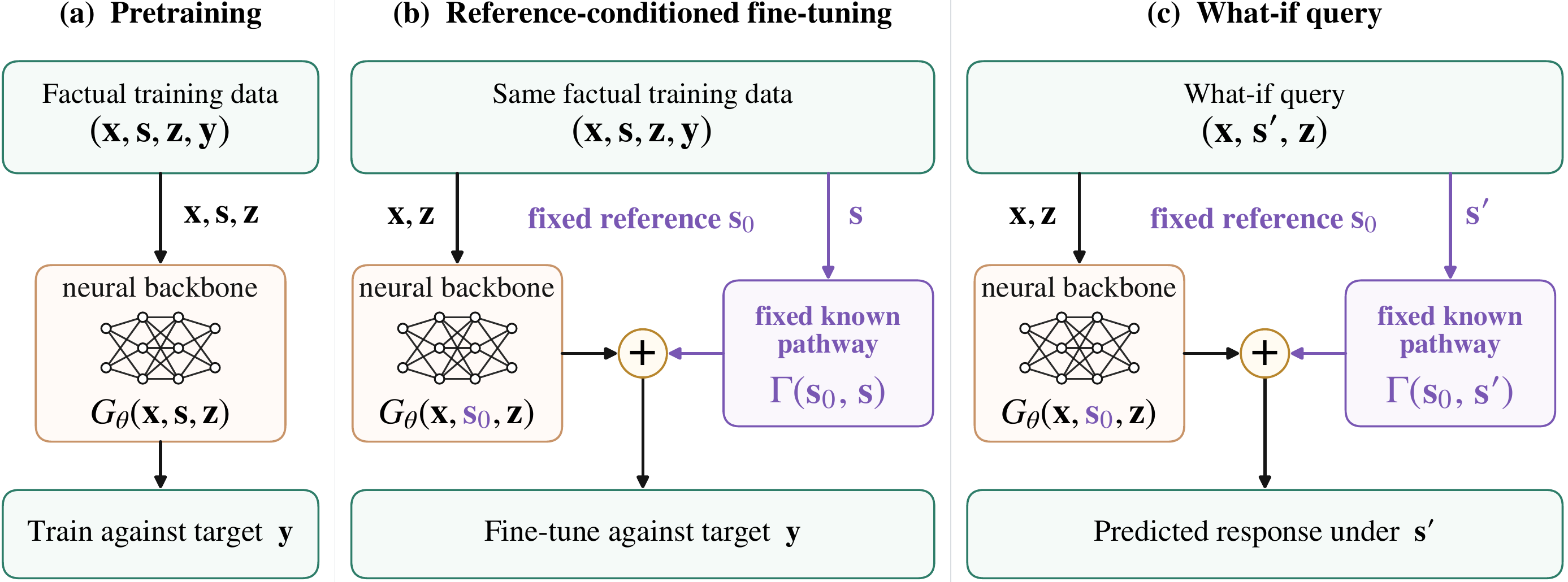}

\caption{Overview of the \emph{ReRoute} framework. \textbf{(a)} Standard training: the emulator $G_{\theta}$ is trained on factual samples $(\mathbf{x},\mathbf{s},\mathbf{z},\mathbf{y})$. \textbf{(b)} \emph{ReRoute} initializes from the pretrained operator $G_{\theta}$, replaces the targeted input $\mathbf{s}$ seen by the backbone with the fixed reference $\mathbf{s}_0$, and reintroduces the factual value $\mathbf{s}$ only through the fixed known pathway $\Gamma(\mathbf{s}_0,\mathbf{s})$. The model is fine-tuned using the same factual data $(\mathbf{x},\mathbf{s},\mathbf{z},\mathbf{y})$. \textbf{(c)} For a what-if query $\mathbf{s}'$, the backbone remains conditioned on $\mathbf{s}_0$, while $\mathbf{s}'$ enters through the pathway $\Gamma(\mathbf{s}_0,\mathbf{s}')$, producing the predicted response under the requested value.}
\label{fig:overview} \vspace{-3mm}
\end{figure}

\looseness=-1 \textbf{Summary of Results.}
\emph{ReRoute} consistently improves counterfactual prediction across controlled and climate CO$_2$ experiments. In advection--diffusion with exact counterfactual responses, it reduces response error by up to $89.6\%$, including over rollouts $10\times$ longer than the fine-tuning horizon. Applied to pretrained climate emulators, it reduces aggregate error by $18.2$--$31.8\%$ under severe SST--CO$_2$ shifts without controlled-CO$_2$ supervision, while maintaining comparable skill under standard and mildly shifted conditions. Finally, in an emulator trained directly on historical ERA5 reanalysis~\citep{hersbach2020era5}, \emph{ReRoute} preserves substantially more near-surface warming when CO$_2$ is held fixed while observed SST and sea-ice evolution are retained, and yields a tropical-cyclone climatology closer to ERA5. These gains require only light fine-tuning and no additional interventional data, at a fraction of the cost of retraining on controlled simulations.

\section{Related Work}
\label{app:related-work}

\looseness=-1
\textbf{Scientific generalization beyond observed regimes.}
A central challenge in scientific machine learning is generalization to conditions that are absent or weakly represented in the training distribution, including unseen interventions, forcing regimes, and system configurations. Existing approaches seek such generalization by learning representations or structural invariances that remain stable across environments~\citep{scholkopf2021toward,yao2025unifying,yu2025causal}, or by relying on sufficient observational support for the target regime~\citep{imbens2015causal,hernan2020causal}. In scientific applications, however, the most interesting queries often deliberately move beyond this support, while collecting new controlled experiments or simulations may be expensive or impossible. \emph{ReRoute} addresses this setting through targeted mechanistic extrapolation: rather than learning a general-purpose representation that must transfer to arbitrary distribution shifts, it uses partial scientific knowledge about how a specific queried variable enters the system and leaves the remaining response to the pretrained model. This replaces broad assumptions about transfer across arbitrary shifts with a narrower, query-specific assumption about how the targeted variable enters the system.

\looseness=-1
\textbf{Hybrid scientific modeling.}
Hybrid scientific models incorporate physical structure through physics-informed objectives~\citep{raissi2019physics,karniadakis2021physics,li2024physics}, dynamical or spectral regularization~\citep{mccabe2023towards,chattopadhyay2024challengeslearningmultiscaledynamics}, hard constraints such as conservation laws~\citep{beucler2019,beucler2021enforcing,Harder2024}, and formulations that combine learned components with known mechanistic dynamics or closures~\citep{pan2018data,rackauckas2020universal}. These approaches typically encode governing equations, conservation relations, or substantial dynamical structure of the system. \emph{ReRoute} uses a more local form of scientific knowledge: it requires only the pathway through which the queried variable first enters the system, while complex downstream interactions and feedbacks remain learned by the pretrained emulator. This makes \emph{ReRoute} a lightweight adaptation strategy for targeted what-if prediction without requiring a complete mechanistic model.

\looseness=-1
\textbf{Climate emulators and forcing disentanglement.}
Our experiments use the ACE2 family of machine-learned climate emulators as a concrete instance of this problem. ACE2 is conditioned on prescribed SST, sea ice, and atmospheric CO$_2$, whose strong co-variation in historical training data can lead to unrealistic responses when these forcings are varied independently~\citep{watt2025ace2}. ACE2-SOM extends the training distribution to substantially different equilibrium CO$_2$ climates using controlled simulations~\citep{clark2025ace2}, while \citet{clark2026disentangling} directly address SST--CO$_2$ entanglement by generating random-CO$_2$ simulations in which the two forcings vary independently. \citet{mahesh2026examining} further show that historically trained weather emulators can reproduce fast radiatively driven responses to unprecedented CO$_2$ perturbations when the relevant response is already represented in historical variability. \emph{ReRoute} targets the complementary setting in which direct manipulation of the learned forcing input is insufficient and additional controlled simulations are unavailable: it supplies the known radiative entry pathway for CO$_2$ while retaining the emulator's learned atmospheric and coupled dynamics, enabling targeted forcing interventions using only the original training distribution.

\section{\emph{ReRoute}: Mechanism-Guided Adaptation}
\label{sec:method}

\looseness=-1 This section formalizes \emph{ReRoute}. We first define the targeted what-if query, then introduce the known pathway and the reference-conditioned fine-tuning procedure. The key idea is simple: \emph{remove the queried driver from the flexible backbone and reintroduce it through a deterministic pathway whose form is known}. The remaining system response is still learned from data.

\begin{takeawaybox}[title={}]
\looseness=-1 \textbf{\emph{ReRoute} in one sentence}: Keep the input variable that we targeted to intervene fixed inside the pretrained backbone, route its variation through the known mechanism, and fine-tune the complete model on the original data.
\end{takeawaybox}

\subsection{Problem Setup: Targeted What-If Queries}

Let $\mathbf{x}\in\mathcal{X}$ denote the current system state. We partition the externally specified inputs as
\[
    \mathbf{u}=(\mathbf{s},\mathbf{z}),
    \qquad
    \mathbf{s}\in\mathcal{S},
    \quad
    \mathbf{z}\in\mathcal{Z},
\]
where $\mathbf{s}$ is the \emph{targeted input variable} that we want to vary in a what-if query, and $\mathbf{z}$ contains the remaining conditions. The system response is
\begin{equation}
    \mathbf{y}
    =
    f^\ast(\mathbf{x},\mathbf{s},\mathbf{z}),
    \qquad
    f^\ast:
    \mathcal{X}\times\mathcal{S}\times\mathcal{Z}
    \rightarrow
    \mathbb{R}^d.
    \label{eq:true-response}
\end{equation}

Training samples come from the distribution available for model training $(\mathbf{X},\mathbf{S},\mathbf{Z},\mathbf{Y})
    \sim P_{\mathrm{train}}$, where $\mathbf{Y}
    =
    f^\ast(\mathbf{X},\mathbf{S},\mathbf{Z})$. Our goal is to answer a targeted query
$f^\ast(\mathbf{x},\mathbf{s}',\mathbf{z})$,
where $\mathbf{s}' \in \Scal$ is a requested value of the targeted driver.     Throughout the paper, we use \textit{counterfactual prediction} to mean predicting the evolution of a realized system under an alternative forcing trajectory while holding the remaining observed conditions fixed. Our causal analysis formalizes the corresponding intervention on the targeted driver at the distributional level.

\begin{intuitionbox}[title={}]
 \textbf{Climate example}: \emph{What if CO$_2$ had stayed at its 1940 level while other boundary conditions evolved as they did historically?} Here, $\mathbf{s}$ is atmospheric CO$_2$, $\mathbf{z}$ contains the remaining control inputs such as SST and sea ice, and $f^\ast$ is the unknown full climate response. Given only the historical trajectory, our goal is to learn an emulator that can answer this unseen fixed-CO$_2$ query.
\end{intuitionbox}

\subsection{A Known Pathway for the Targeted Driver}

Choose a reference value $\mathbf{s}_0\in\mathcal{S}$. We assume access to a fixed pathway
\begin{equation}
    \Gamma (\sb_0, \cdot):
    \mathcal{S}
    \rightarrow
    \mathbb{R}^d,
    \qquad
    \Gamma(\mathbf{s}_0,\mathbf{s}_0)=0,
    \label{eq:fixed-pathway}
\end{equation}
which describes how changing $\mathbf{s}$ away from $\mathbf{s}_0$ directly enters the model response.

Notably, the pathway does \emph{not} need to describe the full system dynamics. It only needs to encode the part of the response that is scientifically known for the targeted driver. For example, in our climate application, $\Gamma$ describes the immediate atmospheric effect of changing CO$_2$, while the subsequent adjustment of the climate state remains learned by the emulator.

\begin{clarificationbox}[title={}]
    \textbf{What does \emph{ReRoute} assume?} \emph{ReRoute} does \emph{not} assume that the complete response to $\mathbf{s}'$ is known. It assumes only that we know how the targeted driver first enters the system. \emph{ReRoute} then lets the learned model propagate that effect through the remaining dynamics.
\end{clarificationbox}

\subsection{Reference-Conditioned Fine-Tuning}

Let $G_{\theta_0}$ denote a pretrained emulator. \emph{ReRoute} initializes $\theta\leftarrow\theta_0$, replaces the targeted input seen by the backbone with the fixed reference $\mathbf{s}_0$, and reintroduces the factual value $\mathbf{s}$ only through the known pathway $\Gamma$:
\begin{equation}
    f_{\theta}
    (\mathbf{x},\mathbf{s},\mathbf{z})
    =
    G_\theta(\mathbf{x},\mathbf{s}_0,\mathbf{z})
    +
    \Gamma(\mathbf{s}_0,\mathbf{s}).
    \label{eq:factual-prediction}
\end{equation}

The \emph{ReRoute}d model is fine-tuned against the original training targets:
\begin{equation}
    \mathcal{L}_{\mathrm{train}}(\theta)
    =
    \mathbb{E}_{P_{\mathrm{train}}}
    \left[
      \left\|
        G_\theta(\mathbf{X},\mathbf{s}_0,\mathbf{Z})
        +
        \Gamma(\mathbf{s}_0,\mathbf{S})
        -
        \mathbf{Y}
      \right\|^2
    \right].
    \label{eq:factual-loss}
\end{equation}

\begin{intuitionbox}[title={}]
\textbf{What is the intuition behind \emph{ReRoute}?} Intuitively, the pretrained backbone can use correlations between $\mathbf{s}$ and other inputs to explain the response. After \emph{ReRoute}, the backbone always sees the same $\mathbf{s}_0$. Variation in the targeted driver can therefore enter the prediction only through $\Gamma$. Fine-tuning lets the backbone adjust to this new division of responsibility.
\end{intuitionbox}

When answering a what-if query, the backbone remains fixed at the same reference condition $\mathbf{s}_0$, while the pathway receives the requested value $\mathbf{s}'$,
as shown in \cref{fig:overview}.

\subsection{Identifiability of targeted interventions}
The reference-conditioned construction above also admits a causal interpretation. Under the structural causal model developed in~\cref{app:proofs}, \emph{ReRoute identifies the interventional response of the targeted variable from observational trajectories and a known first-level mechanism}, without requiring controlled intervention
data. The result relies on the targeted variable entering the downstream system through the supplied mechanistic pathway, observational variation being informative about the downstream response, and this response being sufficiently constrained to determine the requested pathway value.\\

\begin{theorem}[Informal]
\looseness=-1 Let $S$ be the variable targeted by a what-if query, and suppose that its direct effect on the downstream system is fully mediated by a known first-level pathway $\Gamma(\sb_0,\Sb)$. If observational variation in this
pathway identifies the response of the same downstream system, and the responses observed over factual pathway values uniquely determine the response at the target value $\Gamma(\sb_0,\sb')$, then 
    $P\!\left(X \mid \operatorname{do}(\Sb=\sb')\right)$
is identifiable from the observational distribution together with $\Gamma$. 
\end{theorem}
Here, $\operatorname{do}(\Sb=\sb')$ means considering a hypothetical system in which $\Sb$ is directly set to $\sb'$, independently of the factors that normally determine it. This differs from conditioning on $\Sb=\sb'$, which only selects cases where $\Sb$ naturally takes the value $\sb'$ and thus preserves the relationships between $\Sb$ and the factors that influence it. Formal assumptions, statements, and Lean-verified proofs are given in
\cref{app:proofs}.\\

\begin{clarificationbox}[title={}]
\textbf{Why observational data can be sufficient?} Intuitively, the known mechanism changes the identification problem.
Without additional structure, correlated observational drivers may be interchangeable to a flexible predictor: the model can fit the factual distribution while assigning the response to the wrong input. 
\emph{ReRoute} removes this ambiguity at the point where the targeted input $\Sb$ enters the model. The backbone always receives the same reference value $\sb_0$, so variation in $\Sb$ can affect the prediction only through $\Gamma(\sb_0,\Sb)$.
\end{clarificationbox}

Observational trajectories then provide examples of how the unknown downstream system responds to different values of this pathway. \emph{ReRoute} does not need to know those downstream mechanisms explicitly: it learns
their response from the factual data. A what-if query simply evaluates the same learned downstream response at the pathway value corresponding to the
requested intervention $ \Gamma(\sb_0,\sb')$. Thus, the mechanistic pathway supplies the missing attribution, i.e., how the
targeted variable enters the system, while observational data supply the remaining downstream dynamics.

\subsection{\emph{ReRoute} for Autoregressive Systems}
\label{app:autoregressive}

Many scientific emulators, evolve the system autoregressively. Let
    $\mathbf{x}_{t+1}
    =
    f^\ast(\mathbf{x}_t,\mathbf{s}_t,\mathbf{z}_t)$
denote the true one-step transition, where $\mathbf s_t$ is the targeted forcing and $\mathbf z_t$ collects the remaining inputs.
During fine-tuning, \emph{ReRoute} predicts    $\widehat{\mathbf{x}}_{t+1}
    =
    G_\theta(\mathbf{x}_t,\mathbf{s}_0,\mathbf{z}_t)
    +
    \Gamma(\mathbf{s}_0,\mathbf{s}_t),$
    \label{eq:autoregressive-factual}
where the backbone is conditioned on the fixed reference value $\mathbf s_0$, while the observed change in the targeted forcing enters through the mechanistic pathway.
For a what-if query with intervention $\mathbf s'_t$, only the pathway input is changed:
    $\widehat{\mathbf{x}}_{t+1}
    =
    G_\theta(\widehat{\mathbf{x}}_t,\mathbf{s}_0,\mathbf{z}_t)
    +
    \Gamma(\mathbf{s}_0,\mathbf{s}'_t).$
The predicted state is fed back at the next timestep. Thus, the pathway specifies only the first-level response to the intervention; its subsequent effects propagate through the learned autoregressive dynamics.

\section{Experiments and Results}
\label{sec:experiments}

\subsection{Advection--Diffusion Transport}
\label{sec:advection_diffusion}

\looseness=-1 We first evaluate \emph{ReRoute} in a fully controlled dynamical system, where both the governing dynamics and the counterfactual response are known exactly. Advection--diffusion is a canonical model for the transport and mixing of quantities such as heat, pollutants, and atmospheric tracers, and appears broadly across physics, engineering, and environmental science. This provides a clean test of whether \emph{ReRoute} can disentangle strongly correlated controls in a fundamental transport problem before moving to more complex settings. We consider the two-dimensional advection--diffusion equation
on a periodic domain,
\begin{equation}
\frac{\partial u}{\partial t}
+
c\frac{\partial u}{\partial x}
=
D\nabla^2 u,
\end{equation}
where $c$ is the advection velocity and $D$ is the diffusion coefficient. Training trajectories contain strongly correlated $(c,D)$ pairs, with Pearson correlation $r=0.989$, allowing the model to achieve low factual error within this correlated training regime without necessarily learning the effects of $c$ and $D$ independently. We generate 512 training trajectories from 32 initial conditions, each containing
100 transitions on a $64\times64$ grid, with independent validation and test
initial conditions.

Our baseline is a Fourier neural operator (FNO)~\citep{li2020fourier} trained from scratch
on these correlated trajectories. \emph{ReRoute} is initialized from this pretrained FNO and fine-tuned only on the original factual data. Because the dynamics are known analytically, we can construct the exact one-step pathway
\begin{equation}
\Gamma(u_t;c,D)
=
\Phi^h_{c,D}(u_t)
-
\Phi^h_{c_0,D_0}(u_t),
\end{equation}
with reference values $c_0=1$ and $D_0=0.010$. Here, $\Phi^h_{c,D}(u_t)$ denotes the field given by advancing $u_t$ by one interval $h$ with parameters $c$ and $D$. We evaluate both rerouting
$D$ alone and jointly rerouting $(c,D)$. The latter clamps both controls in
the backbone and introduces their joint effect through $\Gamma$. We evaluate 100-step autoregressive rollouts from held-out initial conditions over a $41\times41$ $(c,D)$ grid spanning $[0.5,1.5]\times[0.002,0.018]$, including combinations far from the correlated training support. Our main metric is the RMSE of the
predicted counterfactual response relative to the exact solution. Dataset construction and FNO pretraining are detailed in \cref{app:advection_data}, the exact pathway is derived in \cref{app:advection_pathway}, and the fine-tuning procedure and metric definitions are given in \cref{app:advection_evaluation}.

\looseness=-1 \paragraph{\emph{ReRoute} separates strongly correlated transport controls.}
\Cref{fig:four_grids} shows counterfactual-response error across the full $(c,D)$ parameter grid, extending beyond the correlated parameter combinations seen during training. The pretrained FNO performs well near the narrow
training support but its error grows substantially as the two controls are
varied independently. Rerouting $D$ alone reduces this error even though
advection speed $c$ continues to enter through the learned backbone. Jointly
rerouting $(c,D)$ performs best across most of the grid, showing that \emph{ReRoute}
can recover the effects of multiple correlated controls when their joint
entry mechanism is known.

\paragraph{\textbf{The improvement reflects substantially smaller field-level errors.}}
For the representative out-of-distribution intervention in \cref{fig:three_fields}, the baseline produces a large and spatially
coherent response error after 100 rollout steps. Its RMSE is $0.1182$.
Rerouting $D$ reduces the error to $0.0108$, while jointly rerouting $(c,D)$
reduces it to $0.0079$, corresponding to reductions of $90.8\%$ and
$93.3\%$, respectively. Thus, the aggregate gains are not produced by
cancellation across space: \emph{ReRoute} substantially reduces the local
field-level error itself.

\begin{wrapfigure}{r}{0.5\textwidth} \vspace{-1.5mm}
    \centering
    \includegraphics[width=\linewidth]{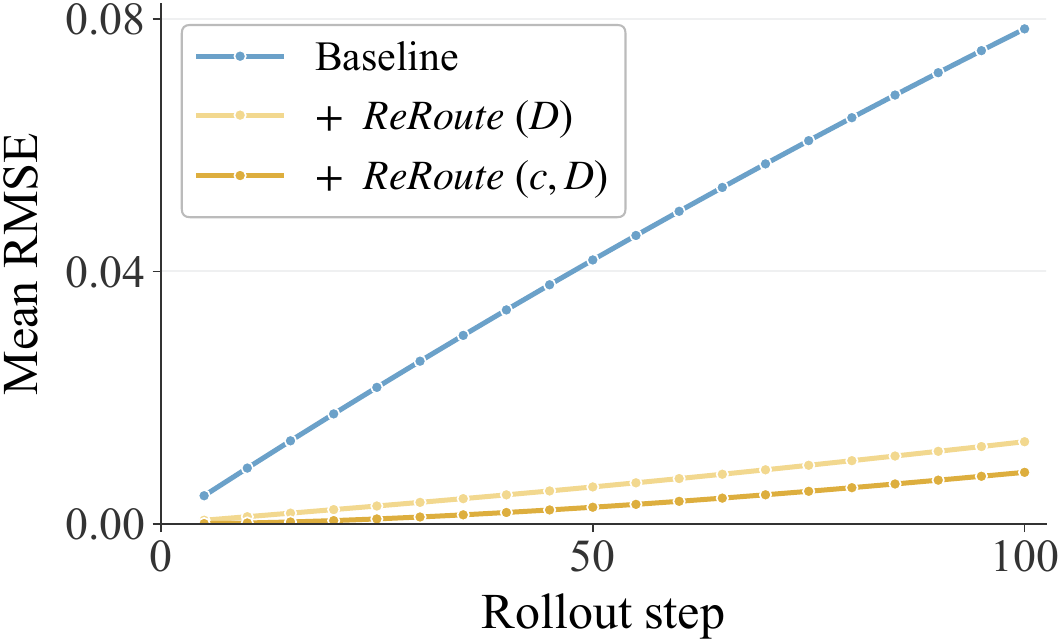}
    \caption{
    \textbf{\emph{ReRoute} remains accurate over long autoregressive rollouts.}
    Counterfactual-response RMSE over 380 held-out parameter pairs far from the training diagonal, averaged across held-out initial conditions. Fine-tuning uses 10-step rollouts; evaluation extends to 100 steps.
    }
    \label{fig:rollout_error}
    \vspace{-1.0em}
\end{wrapfigure}

\looseness=-1 \paragraph{The gains persist far beyond the fine-tuning horizon.}
\Cref{fig:rollout_error} tracks counterfactual response error over the
full 100-step autoregressive rollout, $10\times$ longer than the $10$-step
horizon used during optimization. Both \emph{ReRoute} variants maintain substantially
lower error throughout the rollout. At step 100, the baseline reaches an
average error of $0.0784$, compared with $0.0130$ ($\downarrow 83.4\%$) and $0.0082$ ($\downarrow 89.6\%$) when rerouting $D$ and $(c,D)$, respectively.

\looseness=-1 This experiment also tests two extensions beyond the formal setting of \cref{app:proofs}: the pathway depends on the evolving state and the
targeted input contains multiple controls jointly. Although the current
identification result does not directly cover these cases, the controlled
system shows that both extensions work effectively in practice.

\hfill

\subsection{Climate Emulation}

ACE2 \citep{watt2025ace2} is a family of state-of-the-art six-hour autoregressive climate emulators conditioned on boundary forcings including sea-surface temperature (SST), sea ice, solar forcing, and global-mean CO$_2$. Despite strong predictive skill, the pretrained ACE2 model exhibits unrealistic sensitivity under unseen CO$_2$ perturbations: because historical SST and CO$_2$ co-vary, the model can reproduce the observed climate while attributing changes to the wrong forcing. For example, fixing CO$_2$ removes too much near-surface warming even when historical SST and sea ice are retained \citep{watt2025ace2}. Recent work mitigates this by adding controlled random-CO$_2$ simulations \citep{clark2026disentangling}, however, these controlled experiments are never available in real-world, limiting their application in actual climate re-analysis data \citep{hersbach2020era5}. 
\emph{ReRoute} instead uses a partial radiative mechanism and requires no controlled perturbation trajectories for training. We take $\sb=\mathrm{CO}_2$ as the intervention target and \emph{ReRoute} its effect through a fixed radiative pathway, while conditioning the pretrained backbone on a reference concentration $\sb_0$. We validated \emph{ReRoute} in both large scale coupled atmosphere-ocean dynamics and real-world reanalysis data, showing substantial out-of-distribution performance improvements without additional controlled data and very light fine-tuning. Full experimental details are given in \cref{app:experiment_details}.

\subsubsection{Coupled Climate With Held-Out Counterfactual Evaluation}

\begin{figure}
    \centering
    \vspace{-1.4mm}
    \includegraphics[width=\linewidth]{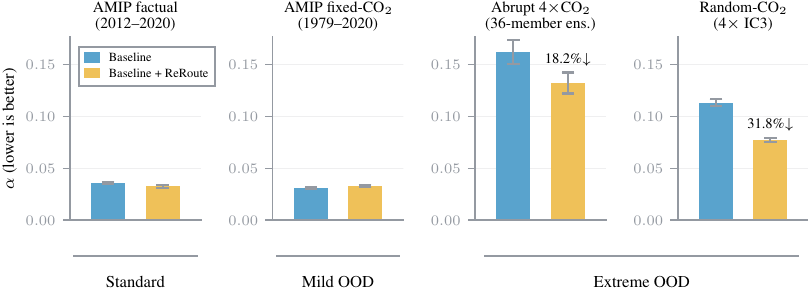}
    \caption{
    \textbf{\emph{ReRoute} becomes increasingly beneficial under severe
    SST--CO$_2$ distribution shift.}
    We compare ACE2S-SHiELD+$_{\mathrm{no-RC}}$ (Baseline) with the same pretrained checkpoint augmented with \emph{ReRoute}. Bars show the mean
    per-realization aggregate climate error $\alpha$, with error bars denoting one standard deviation across independent ensembles.
    Performance remains comparable under standard AMIP and the milder fixed-CO$_2$ intervention, while \emph{ReRoute} reduces error by $18.2\%$
    under abrupt-$4\times$CO$_2$ and by $31.8\%$ under the held-out random-CO$_2$ forcing. The abrupt-$4\times$CO$_2$ experiment uses 36 initial-condition members; the remaining experiments use five stochastic realizations.
    }
    \label{fig:shift-severity}
    \vspace{-1em}
\end{figure}

\looseness=-1
\paragraph{Experiment Setup.}
We use the released ACE2S-SHiELD+$_{\text{no-RC}}$ checkpoint as our baseline throughout this experiment. This model is trained on a mixture of AMIP data (Atmospheric Model Intercomparison Project) and equilibrium-climate simulations. AMIP simulations are essentially atmosphere-only climate simulations in which the ocean surface, including sea-surface temperatures, is prescribed rather than dynamically simulated. Crucially, the model receives no random-CO$_2$ (no-RC) training data in which SST and CO$_2$ are varied independently. \emph{ReRoute} is applied to the same pretrained checkpoint and fine-tuned with the objective function in \cref{eq:factual-loss}. We report the mean and standard deviation of the per-realization channel-averaged normalized spatial RMSE following \citet{clark2026disentangling}, together with the complete training, validation, and evaluation splits (see \cref{app:experiment_details}).

\looseness=-1
\paragraph{Standard climate conditions.}
\emph{ReRoute preserves, and slightly improves factual climate skill under standard forcing conditions.}
We first evaluate the held-out 2012--2020 AMIP factual setting, where SST and CO$_2$ retain the historical relationship represented in training. In this regime, both models achieve low aggregate error, with \emph{ReRoute} slightly improving the mean per-realization $\alpha$ from $0.0359$ to $0.0327$. Thus, rerouting the CO$_2$ input does not require sacrificing the emulator's skill under standard forcing conditions; the adapted model continues to reproduce the factual climate over multi-year autoregressive rollouts.

\looseness=-1
\paragraph{Mild distribution shift.}
\emph{ReRoute remains comparable to the baseline under a mild forcing shift.}
We hold CO$_2$ fixed over the 1979--2020 AMIP rollout while retaining the historical SST evolution, introducing a moderate mismatch between forcings that usually co-vary in training. Here, \emph{ReRoute} obtains $0.0330$ across stochastic realizations, close to the baseline's  $\alpha=0.0311$. The small difference shows that the mechanistic pathway leaves performance largely unchanged when the intervention remains relatively close to the range of climate states represented by the original training data. Together with the factual AMIP experiment, this establishes that \emph{ReRoute} retains the pretrained emulator's climate skill under standard and mildly shifted forcing conditions.

\looseness=-1
\paragraph{Extreme distribution shifts.}
\emph{The mechanistic pathway becomes increasingly valuable as the forcing relationship moves further out of distribution.}
Under abrupt-$4\times$CO$_2$, CO$_2$ changes instantaneously while the atmosphere and slab ocean initially remain in the $1\times$CO$_2$ climate state, creating a forcing--state combination absent from training. Across the 36-member ensemble, \emph{ReRoute} reduces mean per-member aggregate error by $18.2\%$. The improvement is even larger in the held-out random-CO$_2$ experiment, where CO$_2$ varies independently of SST over a wide forcing range: \emph{ReRoute} reduces error by $31.8\%$, from $0.1135$ to $0.0775$. This exposes the failure mode targeted by \emph{ReRoute}: when CO$_2$ departs strongly from the correlations encoded by the observational training distribution, directly conditioning the flexible backbone on CO$_2$ becomes unreliable, whereas routing the perturbation through the known radiative pathway substantially improves the resulting climate response.

\subsubsection{Real-World Climate Reanalysis}

\looseness=-1 \paragraph{Experiment Setup.}
We next apply \emph{ReRoute} to real-world historical reanalysis, where controlled counterfactual simulations are unavailable. We use ACE2-ERA5~\citep{watt2025ace2}, which is trained directly on ERA5 reanalysis, a reconstruction of historical atmospheric conditions obtained by combining weather observations with numerical weather prediction. In these data, SST and CO$_2$ co-evolve rather than varying independently, making their individual effects difficult to disentangle. Methods that rely on randomized or otherwise controlled CO$_2$ trajectories therefore cannot be applied directly. For evaluation, following the fixed-CO$_2$ diagnostic used in \citep{watt2025ace2}, we rerun the historical experiment with CO$_2$ fixed at its 1940 value while retaining the other historical boundary conditions.
We visualize the change in time-mean near-surface temperature relative to the corresponding historical-CO$_2$ rollout,
\begin{equation}
    \Delta T_{2m}=T_{2m}^{\mathrm{fixed\ CO_2}}-
T_{2m}^{\mathrm{historical\ CO_2}}.
\end{equation}
Because the corresponding real-world counterfactual is unobserved, we use this experiment as a \emph{qualitative test} of whether the emulator preserves the surface warming encoded by the prescribed boundary conditions.

\begin{figure}[t] \vspace{-1.4mm}
    \centering
    \includegraphics[width=0.94\linewidth]{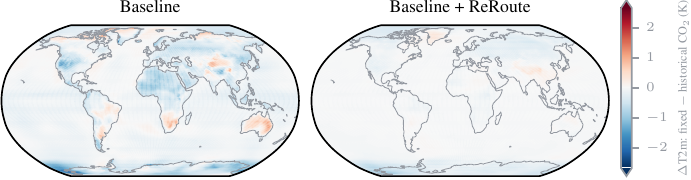}
   \vspace{-1.5mm} \caption{\looseness=-1 \textbf{\emph{ReRoute} preserves surface warming under fixed CO$_2$ in
    ACE2-ERA5, i.e. near-zero values in the heatmap.}
    Shown is the change in time-mean near-surface temperature when CO$_2$ is held fixed at its 1940 value instead of following the historical trajectory, while retaining the other historical boundary conditions such as SST and sea-ice.
    The ACE2-ERA5 baseline exhibits widespread cooling, indicating that much of its historical warming is spuriously tied to the CO$_2$ input.
    With \emph{ReRoute}, the spurious cooling is substantially corrected over most of the globe, showing that the model retains the surface warming driven by the prescribed boundary conditions.}
    \label{fig:res-ace2-era5} \vspace{-3.5mm}
\end{figure}

\looseness=-1
\paragraph{Results.}
\emph{ReRoute preserves substantially more surface warming under fixed CO$_2$ using reanalysis data alone.}
As shown in \cref{fig:res-ace2-era5}, the ACE2-ERA5 baseline exhibits widespread negative $\Delta T_{2m}$ when CO$_2$ is fixed, indicating that much of the historical surface warming disappears despite retaining the same evolving SST and sea-ice boundary conditions, with the strongest response over land and high latitudes. With \emph{ReRoute}, the fixed- versus historical-CO$_2$ difference becomes much smaller over most of the globe, so the emulator retains most of the surface warming even when the historical CO$_2$ trajectory is removed. This setting is particularly important because ERA5 provides no randomized controlled trajectories in which SST and CO$_2$ are varied independently. \emph{ReRoute} instead uses the known CO$_2$ pathway to separate the targeted forcing response while learning the remaining climate dynamics entirely from historical reanalysis. \emph{ReRoute} also improves the generated tropical-cyclone climatology: over the held-out 2001--2010 period, the same TempestExtremes tracker detects 49.8 cyclones yr\(^{-1}\) in the baseline and 43.6 yr\(^{-1}\) with ReRoute, much closer to ERA5's 40.3 yr\(^{-1}\) (\cref{fig:tc-tracks}).

\section{Discussion}
\label{sec:discussion}

\looseness=-1
Predicting how a scientific system responds to an intervention usually requires either intervention data or a complete mechanistic model, both of which are often unavailable in practice. We introduced \textsc{\emph{ReRoute}}, a lightweight framework for targeted scientific what-if queries that combines \emph{observations + mechanism}. \emph{ReRoute} uses partial mechanistic knowledge of how the queried variable enters the system, while observational data identify the downstream interactions, feedbacks, and long-term response. Under explicit assumptions on this first-level pathway, observational variation, and the downstream response class, we show that the targeted interventional distribution is identifiable. By requiring knowledge only of this targeted interface rather than the full governing system, \emph{ReRoute} provides a practical route to mechanism-guided what-if prediction in scientific domains where controlled experiments are costly or impossible and mechanistic knowledge is incomplete.

\looseness=-1
Across controlled and realistic settings, \emph{ReRoute} consistently improves what-if prediction under distribution shift, with the largest gains under the strongest shifts. In controlled advection--diffusion, it cuts counterfactual response error by up to $89.6\%$ far beyond training support, over rollouts $10\times$ longer than the fine-tuning horizon. Under severe SST--CO$_2$ shifts in ACE2, it reduces aggregate error by $18.2$--$31.8\%$ without random-CO$_2$ supervision, while preserving skill under standard and mild shifts. In an ERA5-trained emulator, it substantially restores near-surface warming under fixed CO$_2$ and yields a tropical cyclone climatology closer to ERA5 using historical reanalysis alone. Together, these results show that targeted mechanistic knowledge can deliver robust gains across systems and shifts with light fine-tuning, no counterfactual data, and a fraction of counterfactual retraining cost.

\textbf{Limitations.} The current theory is limited to a single targeted variable with a state-independent pathway, although our experiments indicate that \emph{ReRoute} can remain effective in practice with state-dependent pathways and multiple jointly-rerouted variables. The framework also assumes near-exact knowledge of how the targeted variable enters the system, which can be limiting in applications where the mechanism is only approximately known. While \emph{ReRoute} may remain beneficial under modest pathway misspecification, imperfect mechanistic knowledge could limit its effectiveness.

\subsection*{AI use statement}

Large language models (LLMs) were used to assist with manuscript preparation and experimental presentation, including refining the writing, generating figures, and creating tables. The core research ideas, methodology, experimental design, and scientific conclusions were developed by the authors.

\subsection*{Ethics statement}

This work uses existing climate reanalysis data, physics-based simulations, and pretrained scientific models; it involves no human subjects or sensitive personal data. \emph{ReRoute} is intended for scientific research and should not be treated as a substitute for established physical climate models or as direct operational or policy guidance. Its predictions may inherit limitations from the underlying data, pretrained emulator, and supplied mechanistic pathway, especially under extreme distribution shifts. We therefore report assumptions, controlled evaluations where available, and known failure modes.

\subsection*{Reproducibility statement}

We provide detailed descriptions of the datasets, training and evaluation splits, model checkpoints, pathway construction, fine-tuning procedure, hyperparameters, rollout protocols, ensemble settings, and evaluation metrics in \cref{app:experiment_details}. The appendix also specifies the exact experimental configurations used for each climate benchmark and the computation of all reported metrics. Upon acceptance, we will release the curated implementation of \emph{ReRoute}, together with the training and evaluation configurations and scripts required to reproduce the main results and figures.

\subsubsection*{Acknowledgments}
We thank Pedram Hassanzadeh, Spencer Clark, and Oliver Watt-Meyer for helpful discussions. D.Y. is supported by a Google PhD Fellowship. A.A. is supported by the Bren Endowed Chair, the Office of Naval Research (MURI Grant N00014-23-1-2654), and the AI2050 Senior Fellows Program at Schmidt Sciences. This work was conducted at Caltech during D.Y.'s research visit and V.D.'s postdoctoral appointment.

\bibliography{refs}
\bibliographystyle{refs}

\clearpage

\appendix

\addcontentsline{toc}{section}{Appendix} %
\part{Appendix} %
{
  \hypersetup{linkcolor=pcpurple}
  \parttoc
}

\hfill 

\section{A Causal Perspective of ReRoute}
\label{app:proofs}
This section gives a causal identification analysis of \emph{ReRoute}. We formalize when a target
intervention can be recovered from observational data together with partial mechanistic
knowledge, and make explicit the assumptions required for this recovery.

\subsection{Preliminaries}

We consider an acyclic structural causal model (SCM) with a distinguished
\emph{source variable} $S$ and downstream variables
$X:=(X_1,\ldots,X_N).$
Throughout this section, interventions act on $S$.

\begin{definition}[Structural causal model]\label{def:scm}
An SCM is a tuple
$
\mathcal M=(\mathcal U,\mathcal V,\mathcal F,P_U),
$
where
$
\mathcal U={U_s,U_1,\ldots,U_N},
$
and
$
\mathcal V={S,X_1,\ldots,X_N}
$
are the exogenous and endogenous variables,
$P_U$ is the joint distribution of the exogenous variables, and
$\mathcal F = {f_s, f_1, \dots, f_N}$ specifies the structural mechanisms.
We assume that $S$ is a source node, so
\begin{equation}
\begin{aligned}
S &:= f_s(U_s), \
X_i &:= f_i\bigl(X_{\pa_{\mathcal M}(i)},U_i\bigr),
\qquad i=1,\ldots,N.
\label{eq:prelim-structural-equations}
\end{aligned}
\end{equation}

The associated acyclic graph $G_{\mathcal M}$ contains an edge
$V_k\to V_i$ whenever the structural assignment for $V_i$ depends on $V_k$.
We write $\pa_{\mathcal M}(i)$ for the parents of $X_i$.
\end{definition}

Because the model is acyclic, its structural equations can be solved recursively.
The following parameterized lemma gives the existence, uniqueness, and measurability
needed both for the original SCM and for the modified system introduced in
\cref{def:gf-forced-system}.

\hfill 

\begin{assumption}[Measurability and regularity conventions]
\label{ass:gf-regularity}
Throughout this section we use the following conventions.
\begin{enumerate}[left=1.5em]
\item The source takes values in a measurable space $\mathcal S$, and each downstream
variable is real valued, so that $X$ takes values in $\R^{N}$. Thus the
downstream state space is standard Borel, and regular conditional
distributions of $X$ exist and are unique up to null sets.
\item $P_U$ is a probability measure on the exogenous space. It may be an arbitrary joint
law: the coordinates of $U_X:=(U_1,\ldots,U_N)$ need not be independent, so latent
confounding among the downstream variables is allowed.
\item Every structural mechanism is \emph{jointly} measurable in all of its arguments:
$f_s$ is measurable, and each $f_i$ is measurable on the product of its
parents' value spaces with $\mathcal U_i$.
\item Every mapping declared \emph{known} below, including the source-interface
response $\Gamma(s_0,\cdot)$ of \cref{ass:known-interface}, is measurable.
\end{enumerate}
\end{assumption}

\hfill 

\begin{lemma}[Recursive well-posedness]
\label{lem:gf-wellposed}
Let $G$ be a finite acyclic graph on ${1,\ldots,N}$, let $\Theta$ and
$\mathcal V_1,\ldots,\mathcal V_N$ be measurable spaces, and for each $i$ let

$$
    g_i:\R^{|\pa_G(i)|}\times\mathcal V_i\times\Theta\longrightarrow\R
$$

be jointly measurable. Then for every $\theta\in\Theta$ and every
$v=(v_1,\ldots,v_N)$ the system

$$
    x_i=g_i\bigl(x_{\pa_G(i)},v_i,\theta\bigr),
    \qquad i=1,\ldots,N,
$$

has a unique solution $x=\Psi(\theta,v)\in\R^N$, and the resulting map
$\Psi:\Theta\times\prod_i\mathcal V_i\to\R^N$ is jointly measurable.
\end{lemma}
\begin{proof}
Choose a topological order of $G$. At each vertex, the right-hand side depends
only on variables earlier in the order, so its value is uniquely determined by
values already computed. Recursing through the order gives existence and
uniqueness. The same recursion proves joint measurability, since at each step
the jointly measurable map $g_i$ is composed with previously constructed
jointly measurable maps of $(\theta,v)$.
\end{proof}

\looseness=-1 Apply \cref{lem:gf-wellposed} to
\eqref{eq:prelim-structural-equations}, taking $G$ to be the restriction of
$G_{\mathcal M}$ to the downstream variables and $\Theta=\mathcal S$ to carry
the source value. Then $X$ is a measurable function of $(S,U_X)$, and hence of
$(U_s,U_X)$. We will apply the same lemma in
\cref{def:gf-forced-system} with $\Theta=\mathcal G$. 

\hfill

\textbf{Observational Distribution.}
Sampling $U\sim P_U$ and evaluating the structural equations induces the
observational distribution $P_{S,X}^{\mathcal M}$ of $(S,X)$.
We write $P_X^{\mathcal M}$ for its $X$-marginal. \\

\begin{definition}[Source intervention]\label{def:source-intervention}
Let $S$ be a source node in the SCM (\cref{def:scm}).
For an intervention value $\bar s\in\mathcal S$, the intervention
$do(S=\bar s)$ replaces the structural equation $f_s$ by
$S:=\bar s$,
while leaving all other structural equations and the exogenous distribution
$P_U$ unchanged.
Let $\mathcal M_{do(S=\bar s)}$ denote the resulting intervened SCM. The corresponding
\textit{interventional distribution} of the downstream variables $X$ is
$P_X^{\mathcal M;do(S=\bar s)}$.
\end{definition}

\looseness=-1 \paragraph{Overview.}
Our goal is to identify
$P_X^{\mathcal M;do(S=\bar s)}$
from the observational distribution $P_{S,X}^{\mathcal M}$ and a known
\emph{source-interface response} $\Gamma(s_0,s)$, without specifying the other
downstream causal mechanisms.
The argument has three steps.
First, $do(S=\bar s)$ fixes the known pathway at
$\Gamma_\star:=\Gamma(s_0,\bar s)$,
so the causal target becomes the downstream response at $\Gamma_\star$.
Second, under source exogeneity, observational variation in
$\Gamma(s_0,S)$ identifies the downstream response map for
$\mu$-almost every observed pathway value, where $\mu$ is the law of
$\Gamma(s_0,S)$.
Third, a target-determination condition converts this
$\mu$-almost-everywhere identification into identification at the particular
point $\Gamma_\star$, including when the target lies outside the observational
distribution.
These three steps give identification in the two-model sense of
\cref{def:gf-identifiable}.

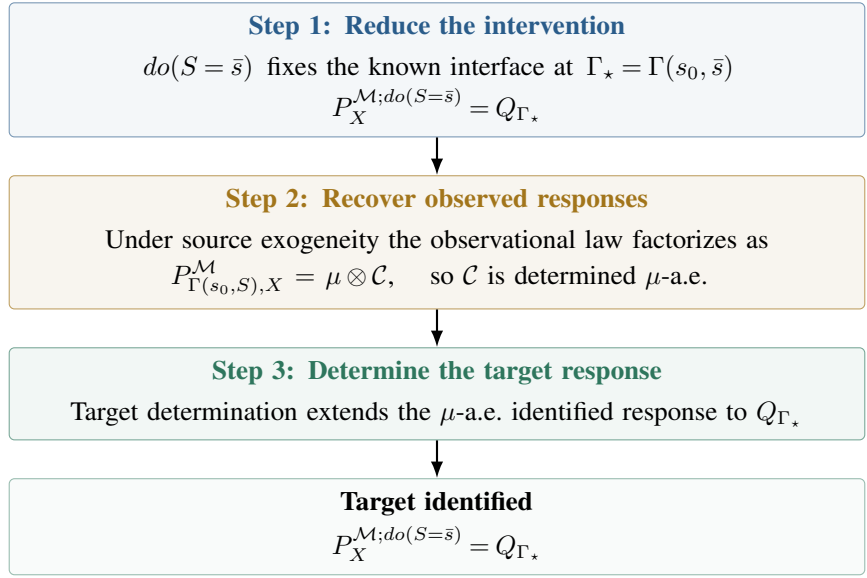
\begin{figure}[t]
\centering
\begin{tikzpicture}[
font=\normalsize,
node distance=5mm,
roadmapbox/.style={
draw,
rounded corners=2pt,
text width=0.78\linewidth,
minimum height=12mm,
align=center,
inner xsep=6pt,
inner ysep=5pt
},
arr/.style={-{Latex[length=2.2mm]},thick}
]

\node[
roadmapbox,
fill=pcblue!5,
draw=pcblue!65
] (s1) {
{\bfseries\color{pcblue} Step 1: Reduce the intervention}\\[1.5mm]
$do(S \!\!\!\!=\!\!\!\!\bar s) \,$ fixes the known interface at $ \, \Gamma_\star \!\!\!\!= \!\!\!\!\Gamma(s_0,\bar s)$\\[0.5mm]
$\displaystyle
P_X^{\mathcal M;do(S=\bar s)} \!\!\!\!\!\!\!\!\!\!= \!\!\!\!\!\!\!\!\!\! Q_{\Gamma_\star}
$
};

\node[
roadmapbox,
fill=pcgold!7,
draw=pcgold!70,
below=of s1
] (s2) {
{\bfseries\color{pcgold} Step 2: Recover observed responses}\\[1.5mm]
Under source exogeneity the observational law factorizes as\\[0.5mm]
$\displaystyle
P^{\mathcal M}_{\Gamma(s_0,S),X} \!\!\!\!= \!\!\!\!\mu \!\!\otimes \!\! \mathcal C
$,
\quad so $\mathcal C$ is determined $\mu$-a.e.
};

\node[
roadmapbox,
fill=pcgreen!6,
draw=pcgreen!65,
below=of s2
] (s3) {
{\bfseries\color{pcgreen} Step 3: Determine the target response}\\[1.5mm]
Target determination extends the $\mu$-a.e.\ identified response to 
$\displaystyle Q_{\Gamma_\star}$
};

\node[
roadmapbox,
fill=pcgreen!3,
draw=pcgreen!50,
below=of s3
] (out) {
{\bfseries Target identified}\\[1mm]
$\displaystyle
P_X^{\mathcal M;do(S=\bar s)} \!\!\!\!\!\!\!\!\!\!=\!\!\!\!\!\!\!\!\!\! Q_{\Gamma_\star}
$
};

\draw[arr] (s1) -- (s2);
\draw[arr] (s2) -- (s3);
\draw[arr] (s3) -- (out);

\end{tikzpicture}

\caption{Roadmap of the identification argument.}
\label{fig:theory-roadmap} \vspace{3mm}
\end{figure}

\hfill 

\subsection{Source-Interface Reduction}

We first reduce an intervention on $S$ to a change at a lower-dimensional
interface whose response is known.

\begin{assumption}[Known source interface]\label{ass:known-interface}
Let

$$
    D:=\ch_{G_{\mathcal M}}(S)\subseteq\{1,\ldots,N\}
$$

be the complete set of direct children of $S$, and fix a reference source
value $s_0\in\mathcal S$. For every $j\in D$, assume that the
\emph{source-interface response}

$$
    \Gamma_j(s_0,s),
$$

is a known measurable function of $s\in\mathcal S$. Collect these responses as
\begin{equation}
\Gamma(s_0,s)
:=
\bigl(\Gamma_j(s_0,s)\bigr)_{j\in D}
\in\mathcal G:=\R^{|D|}
\label{eq:gf-interface}
\end{equation}
with
$
\Gamma(s_0,s_0)=0.
$
\end{assumption}

The reference value $s_0$ enters through the notation $\Gamma(s_0,s)$ and
through the decomposition of the child mechanisms in
\cref{ass:source-additivity}.

\begin{assumption}[Source-additive child mechanism]
\label{ass:source-additivity}
For every direct child $j\in D$, there exists an unknown jointly measurable
mechanism $r_j$ such that
\begin{equation}
X_j
= r_j\left(
X_{\pa_{\mathcal M}(j)\setminus{{S}}},
U_j
\right)
+
\Gamma_j(s_0,S)
.
\label{eq:gf-source-additive-child}
\end{equation}
\end{assumption}

Equivalently, \eqref{eq:gf-source-additive-child} says that for all
$s,s'\in\mathcal S$, all parent configurations $x$, and all disturbances $u$,
\begin{equation}
f_j(x,s,u)-f_j(x,s',u)
=\Gamma_j(s_0,s)-\Gamma_j(s_0,s').
\label{eq:gf-additivity-contrast}
\end{equation}
Thus the absolute contribution of $S$ at the reference value $s_0$ need not be
known; it is absorbed into $r_j$. No additive structure is imposed on the
remaining downstream mechanisms or on how the other parents and $U_j$ enter
$r_j$.

The assumption does require the source contribution to be homogeneous at the
interface. By \eqref{eq:gf-additivity-contrast}, changing $S$ changes $X_j$ by
the same amount for every configuration of the other parents and every value
of $U_j$. Thus the direct source contribution cannot be modified by those
variables, although effect modification can still occur downstream.\\

\begin{remark}[Intuition on $\Gamma$]
\label{rem:gf-interface-interpretation}
By \eqref{eq:gf-additivity-contrast}, $\Gamma_j(s_0,s)$ is the
\emph{direct source contribution} at $X_j$: the unit-level change in $X_j$
caused by moving $S$ from $s_0$ to $s$ while holding the other parents of
$X_j$ and the disturbance $U_j$ fixed.
It is therefore a structural contrast local to the edge $S\to X_j$.
\Cref{ass:source-additivity} requires this contrast to be constant across
units and parent configurations.

This direct contribution is not generally the marginal interventional
contrast at $X_j$. From \eqref{eq:gf-source-additive-child}, assuming the
expectations below exist,
\begin{equation}
\begin{aligned}
&\E\bigl[X_j\mid do(S=s)\bigr]-\E\bigl[X_j\mid do(S=s_0)\bigr] 
=
\bigl(\Gamma_j(s_0,s)-\Gamma_j(s_0,s_0)\bigr) \\
&\qquad\qquad\qquad\qquad\qquad\qquad\qquad\qquad\qquad\qquad
+\E\bigl[r_j\bigl(X^{do(S=s)}_{\pa(j)\setminus\{S\}},U_j\bigr)\bigr]
\\
&\qquad\qquad\qquad\qquad\qquad\qquad\qquad\qquad\qquad\qquad -\E\bigl[r_j\bigl(X^{do(S=s_0)}_{\pa(j)\setminus\{S\}},U_j\bigr)\bigr],
\end{aligned}
\label{eq:gf-total-vs-direct}
\end{equation}
whose first bracket is $\Gamma_j(s_0,s)$ under the normalization of
\cref{ass:known-interface}.
The final two terms need not cancel. For example, another parent of $X_j$ may
itself be a descendant of $S$, such as a second child $k\in D$ with an edge
$k\to j$. The marginal contrast at $X_j$ then contains both the direct
contribution and an indirect contribution through $k$. Only the direct
contribution is assumed known.

The direct and marginal contrasts coincide under the following additional
condition:
\begin{equation}
\text{\emph{(ND)}}\qquad
\text{for every } j\in D,\ \text{no parent of } X_j \text{ other than } S
\text{ is a descendant of } S .
\label{eq:gf-nd}
\end{equation}
Under \eqref{eq:gf-nd},
$X_{\pa(j)\setminus{S}}$ is unaffected by the intervention, so the final two
terms of \eqref{eq:gf-total-vs-direct} cancel and
$\Gamma_j(s_0,s)$ can be obtained from a single-source experiment.
Without \eqref{eq:gf-nd}, $\Gamma$ must instead come from mechanistic knowledge
of the edge $S\to X_j$. None of the results below require
\eqref{eq:gf-nd}.
\end{remark}

\hfill \\

\begin{definition}[Forced downstream system]
\label{def:gf-forced-system}
Fix the reference value $s_0$ from \cref{ass:known-interface}, and let

$$
    U_X:=(U_1,\ldots,U_N),
$$

denote the $X$-associated exogenous variables, distributed as $P_{U_X}$.
For a pathway value $\eta=(\eta_j)_{j\in D}\in\mathcal G$, define the
\emph{forced downstream system} by replacing $\Gamma_j(s_0,S)$ with the
constant $\eta_j$ in every direct-child equation
\eqref{eq:gf-source-additive-child}, while leaving all other structural
equations unchanged:
\begin{equation}
x_j = r_j\bigl(x_{\pa_{\mathcal M}(j)\setminus\{S\}},u_j\bigr)+\eta_j
\quad (j\in D),
\qquad
x_i = f_i\bigl(x_{\pa_{\mathcal M}(i)},u_i\bigr)
\quad (i\notin D).
\label{eq:gf-forced-equations}
\end{equation}

\hfill 

By \cref{lem:gf-wellposed}, applied with parameter space
$\Theta=\mathcal G$, the system \eqref{eq:gf-forced-equations} has a unique
solution for every $\eta\in\mathcal G$ and every $u$. The solution map
\begin{equation}
\Phi:\mathcal G\times\prod_i\mathcal U_i\longrightarrow\R^N,
\qquad
x=\Phi(\eta,u),
\label{eq:gf-forced-system}
\end{equation}
is jointly measurable. With the model's own exogenous variables we write
$X=\Phi(\eta,U_X)$.
\end{definition}

The value $\eta$ may be any point in $\mathcal G$, not only a value of the form
$\Gamma(s_0,s)$. This lets us define the response at the target
$\Gamma_\star$ independently of which pathway values occur observationally.

Neither $\Phi$ nor the downstream causal graph is assumed known.
\Cref{def:gf-forced-system} only states that $\Phi$ exists and is measurable.
For $i\notin D$, the equations in \eqref{eq:gf-forced-equations} do not involve
$S$ because $D=\ch_{G_{\mathcal M}}(S)$ contains all direct children of $S$.

\hfill 

\begin{definition}[Downstream response map]
\label{def:gf-response-map}
For each pathway value $\eta\in\mathcal G$, let $Q_\eta$ denote the distribution
of
\begin{equation}
X=\Phi(\eta,U_X),
\qquad U_X\sim P_{U_X},
\end{equation}
that is, $Q_\eta := \bigl(\Phi(\eta,\cdot)\bigr)_{\#} P_{U_X}$.
The \emph{downstream response map} is
\begin{equation}
\mathcal C:\eta\longmapsto Q_\eta.
\label{eq:gf-response-map}
\end{equation}
Since $\Phi$ is jointly measurable and $P_{U_X}$ is a probability measure,
$\mathcal C$ is a Markov kernel from $\mathcal G$ to $\R^N$.
\end{definition}

Thus $\mathcal C$ describes how the downstream system responds to changes in
the known pathway, without requiring the downstream mechanisms themselves to
be known.

\hfill 

We now reduce the intervention on $S$ to evaluation of this response map.
The next lemma gives the corresponding unit-level representation.

\begin{lemma}[Interface representation]
\label{lem:gf-interface-representation}
Under \cref{ass:known-interface,ass:source-additivity}, the following hold pointwise on the
underlying probability space.
\begin{enumerate}[left=1.5em]
\item In the observational model,
$
X=\Phi\bigl(\Gamma(s_0,S),U_X\bigr).
$
\item In the intervened model $\mathcal M_{do(S=\bar s)}$,
$
X=\Phi\bigl(\Gamma(s_0,\bar s),U_X\bigr).
$
\end{enumerate}
\end{lemma}

\begin{proof}
Let $\sigma$ denote the value of $S$: observationally
$\sigma=f_s(U_s)$, while under $do(S=\bar s)$ it equals the constant
$\bar s$.

Because $D=\ch_{G_{\mathcal M}}(S)$ contains every direct child of $S$, no
equation indexed by $i\notin D$ involves $S$. These equations therefore remain
$x_i=f_i(x_{\pa(i)},u_i)$.
For $j\in D$, \cref{ass:source-additivity} gives
$x_j=r_j(x_{\pa(j)\setminus{S}},u_j)+\Gamma_j(s_0,\sigma)$.
Hence $X$ solves exactly \eqref{eq:gf-forced-equations} with
$\eta=\Gamma(s_0,\sigma)$.
By uniqueness in \cref{lem:gf-wellposed}, the solution is
$\Phi(\Gamma(s_0,\sigma),U_X)$.
\end{proof}

\hfill 

\begin{proposition}[Intervention as pathway evaluation]
\label{prop:gf-intervention-reduction}
For an intervention $do(S=\bar s)$, define the target pathway response
\begin{equation}
\Gamma_\star
:=
\Gamma(s_0,\bar s).
\label{eq:gf-target-pathway}
\end{equation}
Under \cref{ass:known-interface,ass:source-additivity},
\begin{equation}
P_X^{\mathcal M;do(S=\bar s)}
=Q_{\Gamma_\star}.
\label{eq:gf-target}
\end{equation}
\end{proposition}

\begin{proof}
The intervention $do(S=\bar s)$ changes only the structural equation for $S$.
It leaves the exogenous distribution unchanged, so $U_X$ still has law
$P_{U_X}$.
By \cref{lem:gf-interface-representation}(ii),
$$
    X=\Phi(\Gamma_\star,U_X).
$$
Therefore
$$
    P_X^{\mathcal M;do(S=\bar s)}
    =\operatorname{Law}\bigl(\Phi(\Gamma_\star,U_X)\bigr)
    =\bigl(\Phi(\Gamma_\star,\cdot)\bigr)_{\#}P_{U_X}
    =Q_{\Gamma_\star},
$$

where the final equality is \cref{def:gf-response-map}.
\end{proof}

Thus the original causal problem reduces to identifying the downstream
response $Q_{\Gamma_\star}$ from observational data.

\hfill 

\subsection{Observational Identification of the Downstream Response}

We now ask what observational variation in $S$ reveals about the response map
$\eta\mapsto Q_\eta$. The key condition is that the source varies independently
of the disturbances governing the downstream system.\\ 

\begin{definition}[Observed pathway distribution and support]
\label{def:gf-observed-forcing}
The observational pathway value is $\Gamma(s_0,S)$. Its distribution and
support are
\begin{equation}
\mu
:=
P_{\Gamma(s_0,S)}^{\mathcal M}
=\bigl(\Gamma(s_0,\cdot)\bigr)_{\#}P_S^{\mathcal M},
\qquad
\mathcal G_{\mathrm{obs}}
:=
\supp(\mu),
\label{eq:gf-observed-support}
\end{equation}
where $P_S^{\mathcal M}$ is the $S$-marginal of the observational distribution
$P_{S,X}^{\mathcal M}$. The map $\Gamma(s_0,\cdot)$ is measurable by
\cref{ass:known-interface}, so $\mu$ is well defined.
\end{definition}

\hfill 

The set $\mathcal G_{\mathrm{obs}}$ describes the pathway values generated by
observed values of $S$. It is used only to describe the regimes in
\cref{rem:gf-determination-regimes}; no result below depends on it.
The target $\Gamma_\star=\Gamma(s_0,\bar s)$ need not belong to
$\mathcal G_{\mathrm{obs}}$.

\hfill 

\begin{assumption}[Source exogeneity]
\label{ass:gf-exogeneity}
The source disturbance is independent of all downstream disturbances:
\begin{equation}
U_s\perp U_X.
\label{eq:gf-exogeneity}
\end{equation}
\end{assumption}

Since $S=f_s(U_s)$ and $\Gamma(s_0,\cdot)$ is measurable,
\cref{ass:gf-exogeneity} implies
\begin{equation}
S\perp U_X,
\qquad
\Gamma(s_0,S)\perp U_X.
\label{eq:gf-pathway-exogeneity}
\end{equation}
Thus changes in the observed pathway do not change the distribution of the
downstream disturbances. Observational variation in the pathway therefore
reveals the response of the same downstream system.

\hfill 

\begin{lemma}[Observational recovery of the downstream response]
\label{lem:gf-observational-recovery}
Under \cref{ass:known-interface,ass:source-additivity,ass:gf-exogeneity}, the observational
joint law of the pathway value and the downstream variables factorizes through the response
map:
\begin{equation}
P^{\mathcal M}_{\Gamma(s_0,S),,X}
=\mu\otimes\mathcal C,
\label{eq:gf-joint-factorization}
\end{equation}
where $\mu\otimes\mathcal C$ denotes the law on $\mathcal G\times\R^N$ with first marginal
$\mu$ and conditional kernel $\mathcal C$, i.e.
$(\mu\otimes\mathcal C)(A\times B)=\int_A Q_\eta(B),\mu(\mathrm d\eta)$. Equivalently, any regular conditional distribution of $X$ given
$\Gamma(s_0,S)$ satisfies
\begin{equation}
P_{X\mid\Gamma(s_0,S)}^{\mathcal M}
(,\cdot\mid\eta)
=Q_\eta
\qquad
\text{for $\mu$-almost every }\eta\in\mathcal G.
\label{eq:gf-observational-identification}
\end{equation}
Hence the observational distribution $P_{S,X}^{\mathcal M}$ determines the
downstream response map $\eta\mapsto Q_\eta$ up to $\mu$-null sets, and only up to
$\mu$-null sets.
\end{lemma}

\begin{proof}
By \cref{lem:gf-interface-representation}(i),
\begin{equation}
X=\Phi\bigl(\Gamma(s_0,S),U_X\bigr).
\label{eq:gf-observational-representation}
\end{equation}
By \eqref{eq:gf-pathway-exogeneity},
$\bigl(\Gamma(s_0,S),U_X\bigr)$ has product law
$\mu\times P_{U_X}$.
Push this law forward by
$(\eta,u)\mapsto\bigl(\eta,\Phi(\eta,u)\bigr)$.
Using \eqref{eq:gf-observational-representation},

$$
    P^{\mathcal M}_{\Gamma(s_0,S),\,X}
    =
    \bigl((\eta,u)\mapsto(\eta,\Phi(\eta,u))\bigr)_{\#}\bigl(\mu\times P_{U_X}\bigr)
    =
    \mu\otimes\mathcal C ,
$$

because at each fixed $\eta$,
$\bigl(\Phi(\eta,\cdot)\bigr)_{\#}P_{U_X}=Q_\eta$.
This proves \eqref{eq:gf-joint-factorization}.

Since $X$ takes values in the standard Borel space $\R^N$
(\cref{ass:gf-regularity}), disintegration over the first marginal $\mu$ is
unique up to $\mu$-null sets. Hence
\eqref{eq:gf-joint-factorization} is equivalent to
\eqref{eq:gf-observational-identification}.
\end{proof}

Thus observational data determine $\eta\mapsto Q_\eta$ for
$\mu$-almost every pathway value, but not at arbitrary individual points.

The remaining question is whether this almost-everywhere information
determines the response at the target $\Gamma_\star$.
This issue arises even when
$\Gamma_\star\in\mathcal G_{\mathrm{obs}}$: if
$\mu({\Gamma_\star})=0$, an almost-everywhere statement alone does not fix
the value at that point.
We address this with a target-determination assumption on the admissible class
of response maps.

\hfill 

\subsection{Admissible Models, Target Determination, and Identification}

\begin{assumption}[Target determination]
\label{ass:gf-determination}
Let $\mathfrak C$ denote the admissible class of downstream response maps, let
$\mu$ be the observed pathway distribution of \cref{def:gf-observed-forcing}, and let
$\Gamma_\star=\Gamma(s_0,\bar s)$ be the target pathway value. For any two maps
$\mathcal C,\widetilde{\mathcal C}\in\mathfrak C$,
\begin{equation}
Q_\eta=\widetilde Q_\eta
\quad
\text{for $\mu$-almost every }\eta
\qquad\Longrightarrow\qquad
Q_{\Gamma_\star}
=\widetilde Q_{\Gamma_\star}.
\label{eq:gf-determination}
\end{equation}
\end{assumption}

\Cref{ass:gf-determination} is a condition on the triple
$(\mathfrak C,\mu,\Gamma_\star)$, not on $\mathfrak C$ alone.
In the identification theorem below, the two models being compared have the
same pathway law $\mu$, so the condition is evaluated at a common $\mu$.

\hfill 

\begin{definition}[Admissible model]
\label{def:gf-admissible}
Fix the interface data $\bigl(D,\Gamma(s_0,\cdot)\bigr)$ and a class $\mathfrak C$ of
downstream response maps. An SCM $\mathcal M$ is \emph{admissible} if it satisfies
\cref{def:scm,ass:gf-regularity}, satisfies
\cref{ass:known-interface,ass:source-additivity,ass:gf-exogeneity} with that same $D$ and
that same $\Gamma(s_0,\cdot)$, and its downstream response map
$\mathcal C^{\mathcal M}$ of \cref{def:gf-response-map} belongs to $\mathfrak C$. We write
$\mathfrak A=\mathfrak A\bigl(D,\Gamma(s_0,\cdot),\mathfrak C\bigr)$ for the set of
admissible models.
\end{definition}

Admissible models share only the specified interface
$\bigl(D,\Gamma(s_0,\cdot)\bigr)$ and the requirement that their response maps
belong to $\mathfrak C$. They may otherwise have different exogenous spaces,
exogenous distributions, downstream graphs, and downstream mechanisms.

\hfill 

\begin{definition}[Identifiability of the target]
\label{def:gf-identifiable}
The target interventional distribution $P_X^{,\cdot,;do(S=\bar s)}$ is
\emph{identifiable from $\bigl(P_{S,X},\Gamma(s_0,\cdot)\bigr)$ within $\mathfrak A$} if for
every pair $\mathcal M,\widetilde{\mathcal M}\in\mathfrak A$,

$$
    P_{S,X}^{\mathcal M}=P_{S,X}^{\widetilde{\mathcal M}}
    \qquad\Longrightarrow\qquad
    P_X^{\mathcal M;do(S=\bar s)}=P_X^{\widetilde{\mathcal M};do(S=\bar s)} .
$$

\end{definition}

\hfill 

We can now combine the three steps.

\begin{theorem}[Identification of the target interventional distribution]
\label{thm:gf-identification}
Let $\mathfrak C$ be a class of downstream response maps, let $\bar s\in\mathcal S$ be an
intervention value, and let $\mathfrak A$ be the set of models admissible for
$\bigl(D,\Gamma(s_0,\cdot)\bigr)$ and $\mathfrak C$ (\cref{def:gf-admissible}).
Let $P_{S,X}$ be an observational distribution, and let
$\mu=\bigl(\Gamma(s_0,\cdot)\bigr)_{\#}P_S$ be the pathway law it induces. If
\cref{ass:gf-determination} holds for the triple $(\mathfrak C,\mu,\Gamma_\star)$, with
$\Gamma_\star=\Gamma(s_0,\bar s)$, then any two models in $\mathfrak A$ whose observational
distribution equals $P_{S,X}$ induce the same target interventional distribution.

If this holds for every observational distribution realizable within $\mathfrak A$, then the
target is identifiable from $\bigl(P_{S,X},\Gamma(s_0,\cdot)\bigr)$ within $\mathfrak A$ in
the sense of \cref{def:gf-identifiable}.
\end{theorem}

\begin{proof}
Let $\mathcal M,\widetilde{\mathcal M}\in\mathfrak A$ have the same
observational distribution $P_{S,X}$, and let
$\mathcal C,\widetilde{\mathcal C}\in\mathfrak C$ be their downstream response
maps.

\emph{Step 1: the two models induce the same joint law of
$\bigl(\Gamma(s_0,S),X\bigr)$.}
Because both models use the same measurable interface
$\Gamma(s_0,\cdot)$, pushing their common observational distribution
$P_{S,X}$ through
$(s,x)\mapsto\bigl(\Gamma(s_0,s),x\bigr)$ gives
\begin{equation}
P^{\mathcal M}_{\Gamma(s_0,S),,X}
=P^{\widetilde{\mathcal M}}_{\Gamma(s_0,S),,X}
=:\nu .
\label{eq:gf-proof-step0}
\end{equation}
Taking first marginals shows that both models also have the same pathway law,
$\mu^{\mathcal M}=\mu^{\widetilde{\mathcal M}}=:\mu$.
This is the $\mu$ used in \cref{ass:gf-determination}.

\emph{Step 2: the two response maps agree $\mu$-almost everywhere.}
Apply \cref{lem:gf-observational-recovery} to each model. Together with
\eqref{eq:gf-proof-step0}, this gives

$$
    \mu\otimes\mathcal C=\nu=\mu\otimes\widetilde{\mathcal C}.
$$

Because $X$ takes values in the standard Borel space $\R^N$, the
disintegration of $\nu$ over $\mu$ is unique up to $\mu$-null sets. Therefore
\begin{equation}
Q_\eta
=\widetilde Q_\eta
\qquad
\text{for $\mu$-almost every }\eta.
\label{eq:gf-proof-step1}
\end{equation}

\emph{Step 3: the target response is determined.}
Since $\mathcal C,\widetilde{\mathcal C}\in\mathfrak C$,
\cref{ass:gf-determination} and \eqref{eq:gf-proof-step1} imply
\begin{equation}
Q_{\Gamma_\star}
=\widetilde Q_{\Gamma_\star}.
\label{eq:gf-proof-step2}
\end{equation}

\emph{Conclusion.}
By \cref{prop:gf-intervention-reduction},
\begin{equation}
P_X^{\mathcal M;do(S=\bar s)}
=Q_{\Gamma_\star},
\qquad
P_X^{\widetilde{\mathcal M};do(S=\bar s)}
=\widetilde Q_{\Gamma_\star}.
\end{equation}
Combining this with \eqref{eq:gf-proof-step2} shows that the two models have
the same target interventional distribution.

If the target-determination condition holds for the pathway law induced by
every realizable $P_{S,X}$, then every pair of admissible models with the same
observational distribution agrees on the target. This is exactly the
identifiability condition in \cref{def:gf-identifiable}.
\end{proof}

\paragraph{Where the extrapolation content sits.}
Steps 1 and 2 determine the response map only up to $\mu$-null sets.
The passage from this almost-everywhere information to the value at the single
target point $\Gamma_\star$ comes entirely from
\cref{ass:gf-determination}.
There are three relevant regimes.\\

\begin{proposition}[Positive-mass targets]
\label{prop:gf-atom}
If $\mu({\Gamma_\star})>0$, then \cref{ass:gf-determination} holds at $\Gamma_\star$ for
\emph{every} class $\mathfrak C$.
\end{proposition}
\begin{proof}
Suppose $Q_\eta=\widetilde Q_\eta$ for $\mu$-almost every $\eta$.
Then the disagreement set
$A:={\eta: Q_\eta\neq\widetilde Q_\eta}$
is contained in a $\mu$-null measurable set $N$.
If $Q_{\Gamma_\star}\neq\widetilde Q_{\Gamma_\star}$, then
$\Gamma_\star\in A\subseteq N$, so
$\mu({\Gamma_\star})\le\mu(N)=0$, contradicting the assumption.
\end{proof}

\hfill 

\begin{remark}[Three regimes]
\label{rem:gf-determination-regimes}
We consider the three following regimes:
\begin{enumerate}[left=1.5em]
\item \emph{Positive-mass target.}
If $\mu({\Gamma_\star})>0$, then
\Cref{prop:gf-atom} makes \cref{ass:gf-determination} automatic for every
$\mathfrak C$. The target itself occurs with positive observational
probability.

\item \emph{In-support, $\mu$-null target.}
If $\Gamma_\star\in\mathcal G_{\mathrm{obs}}$ but
$\mu(\{\Gamma_\star\})=0$, as is typical when $\mu$ is nonatomic, then
\eqref{eq:gf-proof-step1} does not determine $Q_{\Gamma_\star}$ by itself.
A regularity restriction on $\mathfrak C$ is needed.

Continuity is sufficient. Suppose every
$\eta\mapsto Q_\eta$ in $\mathfrak C$ is continuous under any Hausdorff
topology on laws over $\R^N$ that separates distinct laws, such as the
topology of weak convergence. If $Q$ and $\widetilde Q$ differed at
$\Gamma_\star$, disjoint neighborhoods of their two values would pull back
to an \emph{open} neighborhood of $\Gamma_\star$ on which
$Q\neq\widetilde Q$. Because
$\Gamma_\star\in\supp(\mu)$, this neighborhood has positive
$\mu$-measure, contradicting \eqref{eq:gf-proof-step1}.
Other restrictions on $\mathfrak C$ may also give target determination.

\item \emph{Out-of-support target.}
If $\Gamma_\star\notin\mathcal G_{\mathrm{obs}}$, then
\cref{ass:gf-determination} is an extrapolation assumption.
The observational distribution alone cannot verify it; it must follow from
the model class $\mathfrak C$ or from additional mechanistic structure.

\end{enumerate}
\end{remark}

\paragraph{Connection to learning.}
The identification result gives a direct learning problem.
Observational samples provide pairs
$$
    \bigl(\Gamma(s_0,S_r),X_r\bigr),
$$
from which a model learns the downstream response
$\eta\mapsto Q_\eta$.
The counterfactual prediction is then obtained by evaluating this response at $\Gamma_\star=\Gamma(s_0,\bar s).$ These observations are ordinary factual pairs, not unit-level counterfactual
pairs. By \cref{lem:gf-observational-recovery}, they determine the response map
only $\mu$-almost everywhere. Generalization to $\Gamma_\star$ therefore comes
from the structural restrictions encoded by the admissible class
$\mathfrak C$, exactly as formalized by \cref{ass:gf-determination} and the
three regimes in \cref{rem:gf-determination-regimes}.

\subsection{Lean Verification}
\label{app:lean-verification}

We formalized and machine-checked the core identification argument in Lean~4 using \texttt{mathlib} and \emph{Causalean}, the causal-inference
library distributed with the CausalSmith project
\citep{causalsmith}.
The development comprises roughly nine hundred lines of Lean across four modules. It covers the structural reduction from an intervention on the source variable to a forced pathway value, observational recovery of the downstream response under source exogeneity, and the final target-identification theorem.

\paragraph{Correspondence with the mathematical development.}
The Lean development follows the mathematical structure above, but the relevant objects are constructed explicitly rather than introduced only through notation. In particular, the recursive downstream solution, the forced system, and the induced response kernels are all defined as Lean objects and their required measurability properties are proved.

\begin{center}
\small
\setlength{\tabcolsep}{4pt}
\renewcommand{\arraystretch}{1.08}

\begin{tabularx}{\linewidth}{@{}p{0.365\linewidth}X@{}}
\toprule
\rowcolor{gray!10}
\textbf{Mathematical object or result}
&
\textbf{Lean counterpart}
\\
\midrule

\raggedright SCM structure and realized probability model
&
\texttt{Mechanisms}, \texttt{Realized}
\\

\raggedright Recursive solution and \cref{lem:gf-wellposed}
&
\texttt{Mechanisms.solveOffset}, \, 
\texttt{solveOffset\_\allowbreak spec}, \, 
\texttt{eq\_\allowbreak solveOffset}, \, 
\texttt{measurable\_\allowbreak solveOffset}
\\

\raggedright Known source interface, \cref{ass:known-interface}
&
\texttt{Interface}
\\

\raggedright Source-additive structure, \cref{ass:source-additivity}
&
encoded in \texttt{Mechanisms.solve\_\allowbreak spec}
\\

\raggedright Source intervention, \cref{def:source-intervention}
&
\texttt{Realized.doSource},
\texttt{Realized.doLaw}
\\

\raggedright Forced downstream system, \cref{def:gf-forced-system}
&
\texttt{Mechanisms.forced}
\\

\raggedright Downstream response map, \cref{def:gf-response-map}
&
\texttt{Realized.response}
\\

\raggedright Interface representation,
\cref{lem:gf-interface-representation}
&
\texttt{Mechanisms.solve\_\allowbreak eq\_\allowbreak forced\_\allowbreak pathway}
\\

\raggedright Intervention reduction,
\cref{prop:gf-intervention-reduction}
&
\texttt{Realized.doLaw\_\allowbreak eq\_\allowbreak response\_\allowbreak pathway}
\\

\raggedright Source exogeneity, \cref{ass:gf-exogeneity}
&
\texttt{Realized.exog}
\\

\raggedright Observational recovery,
\cref{lem:gf-observational-recovery}
&
\texttt{Realized.pathwayJoint\_\allowbreak eq\_\allowbreak compProd},
\texttt{Realized.condDistrib\_\allowbreak eq\_\allowbreak response}
\\

\raggedright Target determination, \cref{ass:gf-determination}
&
\texttt{TargetDetermines}
\\

\raggedright Identification theorem, \cref{thm:gf-identification}
&
\texttt{reroute\_\allowbreak identification}
\\

\raggedright Positive-mass target, \cref{prop:gf-atom}
&
{\footnotesize
\texttt{targetDetermines\_of\_measure\_singleton\_ne\_zero}}
\\

\raggedright Continuous in-support target,
\cref{rem:gf-determination-regimes}(ii)
&
{\footnotesize
\texttt{targetDetermines\_of\_continuous\_of\_mem\_support}}
\\

\bottomrule
\end{tabularx}
\end{center}

\paragraph{Role of the formalization.}
The structural part of the argument was implemented directly. In particular, the forced downstream system is defined recursively, with existence, uniqueness, and measurability established explicitly. This is important because replacing the source contribution by a pathway value does not by itself define a measurable downstream map. The resulting construction makes \cref{prop:gf-intervention-reduction} a theorem about the SCM rather than a notational substitution.

For the measure-theoretic steps, we reuse existing results from
\texttt{mathlib} and Causalean. Under source exogeneity, the joint law of the observed pathway value and downstream state factorizes through the response kernel, yielding
\cref{lem:gf-observational-recovery}. The main identification proof then uses almost-everywhere uniqueness of conditional kernels to show that two observationally equivalent admissible models have downstream response maps that agree $\mu$-almost everywhere. The target-determination assumption upgrades this agreement to the particular target value $\Gamma_\star$, after which \cref{prop:gf-intervention-reduction} gives equality of the target interventional distributions.

\clearpage

\section{Advection–Diffusion Experiment}
\label{app:advection_diffusion}
This section includes additional details regarding the controlled Advection–Diffusion experiments.

\subsection{Dataset and Pretrained FNO Backbone}
\label{app:advection_data}

We use a two-dimensional advection--diffusion system to demonstrate that \emph{ReRoute} extends beyond the climate setting. This isolates two generalizations: the pathway depends on the evolving state, and the targeted input can contain either two controls jointly.  Unlike the climate application, both the exact dynamics and exact pathway are available here. So, this experiment tests the \emph{ReRoute} construction directly against known ground truth. 

\paragraph{Dataset.} 

We consider a scalar field $u(x,y,t)$ on the periodic domain $[0,2\pi)^2$ following
$$
\frac{\partial u}{\partial t}
+
c\frac{\partial u}{\partial x}
=
D\nabla^2u,
$$
where $c$ is the advection speed in the $x$ direction and $D$ is the diffusivity. 

We discretize the field on a $64\times64$ grid and evaluate the exact Fourier solution every $h=0.01$ time units, retaining 101 states corresponding to 100 transitions. Initial conditions are smooth, zero-mean, unit-RMS Gaussian random fields containing modes through wavenumber six.

The training distribution contains strongly correlated combinations of $c$ and $D$. For each initial condition, we draw 16 stratified values $q\in[-1,1]$ and independent noise $\epsilon\sim\mathcal{U}(-0.15,0.15)$. We then set $c=1+0.5q$ and $D=0.010+0.008\,\operatorname{clip}(q+\epsilon,-1,1)$. Each initial condition receives a newly sampled set of parameter pairs, producing 512 unique training trajectories across 32 initial conditions, with parameter correlation $r=0.989$. The validation and test sets each contain eight independent initial conditions. The sampled training parameters are visualized in~\cref{fig:four_grids}.

\paragraph{Pretraining.}

We pretrain a Fourier neural operator (FNO) \citep{li2020fourier} from scratch on these correlated trajectories and use it as the baseline throughout the experiment. The FNO is a prominent example of a neural operator, a class of models that learn mappings between function spaces and are therefore particularly well suited for approximating PDE solution operators~\citep{kovachki_neural_2023,berner2026principled}. The FNO architecture composes integral operator layers parameterized in Fourier space~\citep{duruisseaux2025FNOGuide}, enabling efficient global interactions through fast Fourier transforms on uniform grids. We use the standard FNO implementation provided in the \texttt{neuraloperator} library~\citep{kossaifi2024neural}.

In our experiments, the FNO receives the current scalar field $u_t$, together with $c$ and $D$, each repeated across the spatial grid as an additional input channel, and predicts a residual update around the current state. The model has four FNO blocks, 64 hidden channels, and 16 Fourier modes along each spatial dimension. We pretrain the backbone for 30 epochs using a 10-step autoregressive factual objective. We evaluate the models over 100-step rollouts, ten times the optimization horizon. On the 128 held-out in-distribution test trajectories, the pretrained FNO has mean factual RMSE $0.00184$ after 10 steps and $0.01448$ after 100 steps.

\subsection{State-Dependent Pathway}
\label{app:advection_pathway}

\paragraph{Exact state-dependent pathway.}

For this system, the exact one-step update is known. Advection translates the field by a distance $ch$ in the $x$ direction, while diffusion progressively smooths its smaller-scale structure. In Fourier space, these effects correspond to a phase shift and wavenumber-dependent damping:
$$
\mathcal{F}[\Phi^h_{c,D}(u_t)](\mathbf{k})
=
\exp\left[
-hD|\mathbf{k}|^2
-ihck_x
\right] \,
\mathcal{F}[u_t](\mathbf{k}),
$$
where $\Phi^h_{c,D}(u_t)$ denotes the field given by advancing $u_t$ by one interval $h$ with parameters $c$ and $D$, and $\mathcal{F}$ denotes the Fourier transform.

We choose $c_0=1$ and $D_0=0.010$, the centers of their respective training ranges, as the reference values and define the pathway as
$$
\Gamma(u_t;c,D)
=
\Phi^h_{c,D}(u_t)
-
\Phi^h_{c_0,D_0}(u_t).
$$

The pathway therefore gives the exact change in the next state caused by moving from the reference parameters to the requested parameters. Note that this is state-dependent, as the change is recomputed from the current predicted field at every rollout step. When rerouting both controls, we clamp both $c$ and $D$ to their reference values in the backbone and introduce their joint effect through the pathway:
$$
\widehat{u}_{t+1}
=
G_\theta(u_t,c_0,D_0)
+
\Gamma(u_t;c,D).
$$
We also evaluate a one-variable version that fixes and reroutes only $D$, while the backbone continues to receive the requested value of $c$:
$$
\widehat{u}_{t+1}
=
G_\theta(u_t,c,D_0)
+
\Phi^h_{c,D}(u_t)
-
\Phi^h_{c,D_0}(u_t).
$$

Since the exact update is available, the pathway itself introduces no approximation error when rerouting either a single control or the joint target $(c,D)$, allowing us to isolate whether the fine-tuned backbone uses this pathway effectively. \\

\subsection{Evaluation and Additional Results}
\label{app:advection_evaluation}

We evaluate the models on a $41\times41$ grid spanning $c\in[0.5,1.5]$ and $D\in[0.002,0.018]$, breaking the strong dependence present during training. For each of the eight held-out initial conditions, the exact solution gives trajectories under both the reference and intervention parameters. For the rollout curve, we average over the 380 grid points satisfying $|\tilde c-\tilde D|/\sqrt{2}\geq0.75$, where $\tilde c=(c-1)/0.5$ and $\tilde D=(D-0.010)/0.008$ are normalized coordinates. No trajectories from this grid are used for training or fine-tuning. \\

Our primary metric is the counterfactual-response RMSE,
$$
\mathcal{E}_{\mathrm{cf}}(t;c,D)
=
\operatorname{RMSE}
\left[
\left(
\widehat{u}^{\,c,D}_t
-
\widehat{u}^{\,c_0,D_0}_t
\right)
-
\left(
u^{c,D}_t
-
u^{c_0,D_0}_t
\right)
\right].
$$
This removes errors shared by the factual and counterfactual trajectories to directly measures whether each model predicts the change caused by the intervention.\\

\paragraph{Results.}

\Cref{fig:four_grids} evaluates whether the models can separate $c$ and $D$ after observing them only in strongly correlated pairs. The baseline error increases away from the narrow training support. Rerouting $D$ reduces this error even though $c$ continues to enter through the backbone, while jointly rerouting $(c,D)$ gives the lowest error across most of the grid.

\Cref{fig:three_fields} shows the spatial response residual for a representative intervention far from the training distribution. The baseline produces large, spatially coherent errors. Some residual structure remains under both \emph{ReRoute} variants, but its magnitude is more than an order of magnitude smaller. So, the aggregate improvements comes from substantially weaker field-level errors rather than cancellation across space. 

Finally, \cref{fig:rollout_error} follows the three methods over a 100-step autoregressive rollout. Both \emph{ReRoute} variants maintain substantially lower error at every evaluated step, with error growing most slowly when $(c,D)$ are rerouted jointly. The improvement therefore extends well beyond the ten-step horizon used during fine-tuning.

The state-dependent pathways and joint rerouting of multiple controls lie outside the current formalization of \emph{ReRoute}, so the theoretical guarantees do not directly hold here. Nevertheless, both extensions work in practice here, and jointly rerouting $(c,D)$ performs best because $c$ translates spatial structure while $D$ damps its smaller scales. Rerouting only $D$ still leaves the baseline to extrapolate the advection response from strongly correlated training data.

\begin{figure}[tbp] \vspace{-1.4mm}
    \centering
    \includegraphics[width=0.92\linewidth]{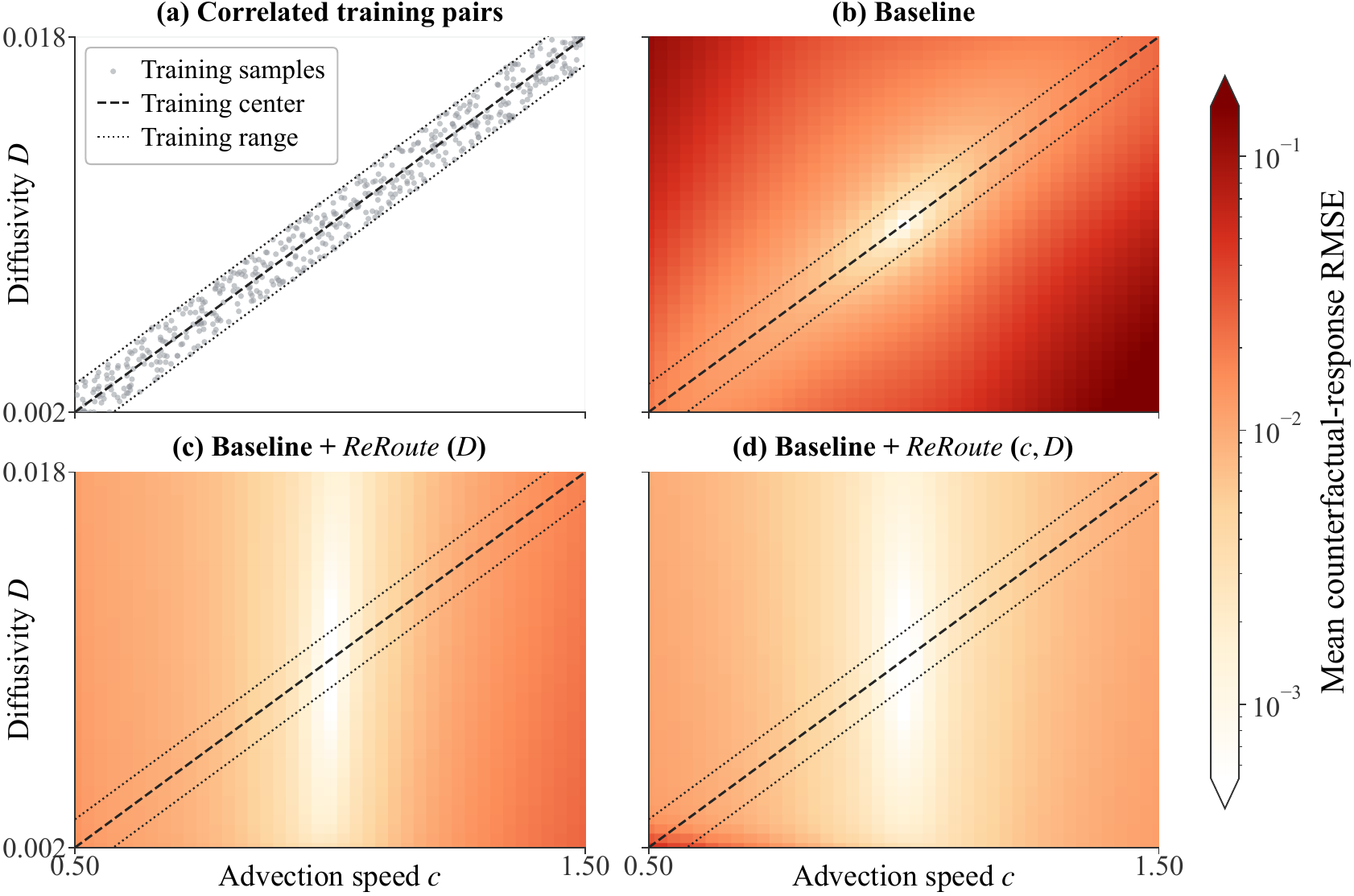} \vspace{-2.5mm}
    \caption{
    \textbf{\emph{ReRoute} separates strongly correlated advection and diffusion controls across the full parameter grid.}
    The left panel shows the 512 parameter pairs used for training; dashed and dotted lines denote their center and range. The remaining panels show mean counterfactual-response RMSE after 100 rollout steps for the pretrained baseline, \emph{ReRoute} applied only to diffusivity $D$, and \emph{ReRoute} applied jointly to $(c,D)$. Errors are averaged across eight held-out initial conditions. \emph{ReRoute} reduces error away from the training support, with the joint pathway giving the lowest errors.
    }
    \label{fig:four_grids} \vspace{-4mm}
\end{figure}

\begin{figure}[tbp]
    \centering
    \includegraphics[width=0.97\linewidth]{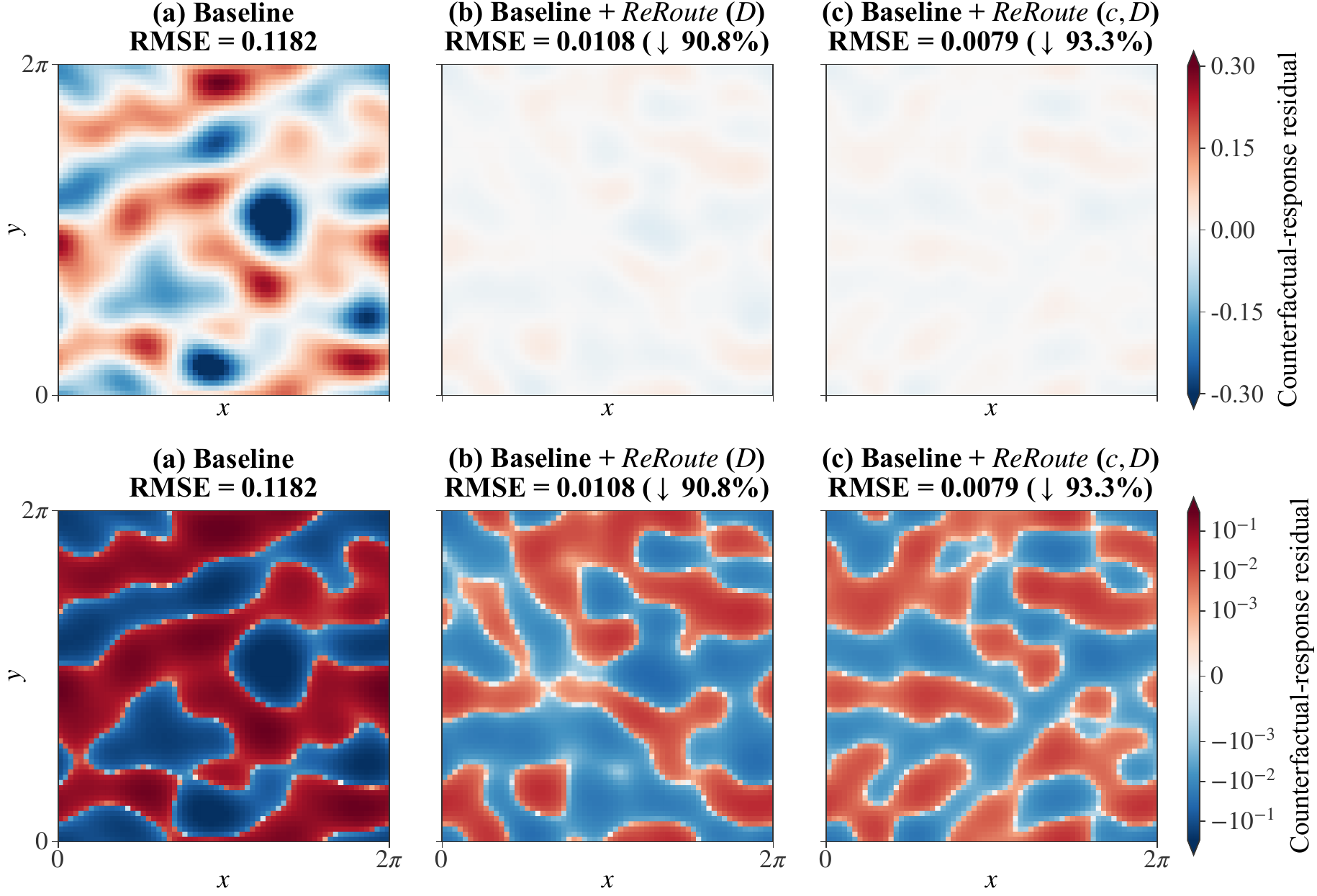} \vspace{-4mm}
    \caption{
    \textbf{\emph{ReRoute} removes most of the structured counterfactual-response error for a representative held-out field.}
    Shown is the signed response residual (\textbf{in linear and log scales}) after 100 rollout steps for one held-out initial condition at $c=0.5$ and $D=0.018$, a parameter combination far from the correlated training support. The baseline has RMSE $0.1182$. Rerouting $D$ alone reduces this to $0.0108$, while jointly rerouting $(c,D)$ reduces it to $0.0079$. 
    }
    \label{fig:three_fields} \vspace{-1.4mm}
\end{figure}

\clearpage

\section{Climate Data and Pretrained Models}
\label{app:data}

We use two pretrained ACE2-family climate emulators in the same order in \cref{sec:experiments}. The controlled climate benchmark uses ACE2S-SHiELD+$_{\mathrm{no\text{-}RC}}$
\citep{clark2026disentangling}, trained on physics-based SHiELD simulations without random-CO$_2$ data. The reanalysis experiment uses ACE2-ERA5 \citep{watt2025ace2}, trained directly on historical ERA5 reanalysis. Both models evolve the atmosphere autoregressively at six-hour resolution on a $1^\circ$ Gaussian grid with eight atmospheric
layers. This section summarizes the training data and variable interface of the two pretrained backbones; intervention and rollout protocols are given in~\cref{app:experiment_details}.

\hfill

\subsection{ACE2S-SHiELD+$_{\mathrm{no\text{-}RC}}$: Controlled Climate Training Data}
\label{app:data_shield}

For the controlled climate experiments, we use the released
ACE2S-SHiELD+$_{\mathrm{no\text{-}RC}}$ checkpoint from
\citet{clark2026disentangling}. The underlying atmospheric simulations
are generated with the GFDL SHiELD model at C96 resolution
(approximately 100~km) with 79 finite-volume hybrid
sigma--pressure levels. SHiELD output is regridded to a
$1^\circ$ Gaussian grid, spherical-harmonic round-trip filtered, and
vertically coarsened to eight atmospheric layers before being used by
ACE2S.

Crucially, the no-RC checkpoint is trained with \emph{no random-CO$_2$
simulations}. Its 90-year training set contains 45 years of AMIP
simulations and 45 years of equilibrium-climate simulations, as
summarized in~\cref{tab:norc_training_data}. The AMIP simulations use
prescribed historical SST, sea ice, and CO$_2$. The equilibrium data
span 1$\times$, 2$\times$, and 4$\times$ the reference CO$_2$
concentration. Although the equilibrium reference simulations originate
from SHiELD coupled to a slab ocean model, SST is prescribed when these
data are used for ACE2S training. Thus, the model is exposed to several
equilibrium CO$_2$ climates but receives no training trajectories in
which CO$_2$ is varied independently of the corresponding climate state.

\hfill 

\begin{table*}[h]
\centering
\caption{
Training and validation data for
ACE2S-SHiELD+$_{\mathrm{no\text{-}RC}}$.
Time bounds are inclusive; IC denotes initial-condition ensemble member.
}
\label{tab:norc_training_data}

\small
\setlength{\tabcolsep}{6pt}
\renewcommand{\arraystretch}{1.15}

\begin{tabularx}{0.94\textwidth}{@{}p{0.11\textwidth}Xr@{}}
\toprule
\rowcolor{gray!10}
\textbf{Split} & \textbf{Simulation samples} & \textbf{Years} \\
\midrule

\multicolumn{3}{@{}l}{\textit{Training}}\\
\addlinespace[2pt]

AMIP
&
IC1: 1979--2008;
IC2: 1981-07--1986-06,
1991-07--1996-06,
2001-07--2006-06
&
45
\\

Equilibrium
&
\begin{tabular}[t]{@{}l@{\qquad}l@{\qquad}l@{}}
1$\times$CO$_2$: & IC1: 2031--2040, & IC2: 2032--2036 \\
2$\times$CO$_2$: & IC1: 2031--2040, & IC2: 2032--2036 \\
4$\times$CO$_2$: & IC1: 2031--2040, & IC2: 2032--2036
\end{tabular}
&
45
\\

\midrule

\multicolumn{3}{@{}l}{\textit{Validation}}\\
\addlinespace[2pt]

AMIP
&
IC1: 2009--2011
&
3
\\

Equilibrium
&
IC2: 2038 at each of
1$\times$, 2$\times$, and 4$\times$CO$_2$
&
3
\\

\bottomrule
\end{tabularx}
\end{table*}

ACE2S-SHiELD+ uses essentially the same prognostic, forcing, and
diagnostic interface as ACE2-ERA5, summarized in
\cref{tab:ace2_variables}. There are two relevant differences.
First, ACE2S-SHiELD+ additionally predicts total frozen precipitation
rate as an output-only diagnostic, which enters the atmospheric energy
budget. Second, CO$_2$ is supplied as a next-timestep forcing, so the
specified concentration corresponds to the six-hour interval over which
the state update and associated boundary-energy fluxes are computed.\\

\subsection{ACE2-ERA5: Historical Reanalysis Training Data}
\label{app:data_era5}

For the real-world reanalysis experiment, we use ACE2-ERA5
\citep{watt2025ace2}, trained directly on ERA5
\citep{hersbach2020era5}. The processed dataset spans 1940--2022.
Training uses 1940--1995, 2011--2019, and 2021--2022;
1996--2000 is used for validation; and 2001--2010 and 2020 are held
out for testing. These are the splits used by the released ACE2-ERA5
checkpoint.

ERA5 is natively represented on 137 model levels. Following
\citet{watt2025ace2}, the fields are horizontally regridded to a
$1^\circ$ Gaussian grid and vertically coarsened to eight
terrain-following hybrid sigma--pressure layers. Air temperature,
specific total water, and horizontal winds are represented on each
vertical layer; the remaining state, forcing, and diagnostic fields are
two-dimensional. The temporal resolution is six hours. State variables
are represented as instantaneous snapshots, while radiative, turbulent,
and precipitation fluxes are means over the six-hour interval.

\Cref{tab:ace2_variables} lists the ACE2-ERA5 variable interface.
Prognostic variables are supplied as inputs and predicted at the next
timestep, forcing variables are input only, and diagnostic variables are
predicted outputs that are not autoregressively fed back into the state.
Surface skin temperature has different roles by surface type: over land
and sea ice it is prognostic, whereas over open ocean it is prescribed
and corresponds to sea-surface temperature (SST).

\begin{table*}[t]
\centering
\caption{
ACE2-ERA5 variable interface, which is also shared by
ACE2S-SHiELD+ except for the differences described in
\cref{app:data_shield}. For vertically resolved variables,
$k=0,\ldots,7$ indexes the eight atmospheric layers from the top of the
atmosphere toward the surface. ``Snapshot'' denotes an instantaneous
field, ``Mean'' a mean over the six-hour timestep, and ``Invariant'' a
time-independent field. TOA denotes top of atmosphere.
}
\label{tab:ace2_variables}
\small
\setlength{\tabcolsep}{8pt}
\renewcommand{\arraystretch}{1.1}
\begin{tabular}{@{}llll@{}}
\toprule
\rowcolor{gray!10}
\textbf{Symbol} & \textbf{Description} & \textbf{Units} & \textbf{Time} \\
\midrule

\multicolumn{4}{@{}l}{\textit{Prognostic variables (input and output)}}\\
\midrule

$T_k$              & Air temperature                         & K                         & Snapshot \\
$q^T_k$            & Specific total water (vapor + condensates) & kg\,kg$^{-1}$          & Snapshot \\
$u_k$              & Eastward wind                           & m\,s$^{-1}$               & Snapshot \\
$v_k$              & Northward wind                          & m\,s$^{-1}$               & Snapshot \\
$T_s$              & Skin temperature over land or sea ice   & K                          & Snapshot \\
$p_s$              & Surface pressure                        & Pa                         & Snapshot \\
$T_{\mathrm{2m}}$  & 2-m air temperature                     & K                          & Snapshot \\
$q_{\mathrm{2m}}$  & 2-m specific humidity                   & kg\,kg$^{-1}$              & Snapshot \\
$u_{\mathrm{10m}}$ & 10-m eastward wind                      & m\,s$^{-1}$                & Snapshot \\
$v_{\mathrm{10m}}$ & 10-m northward wind                     & m\,s$^{-1}$                & Snapshot \\

\midrule
\multicolumn{4}{@{}l}{\textit{Forcing variables (input only)}}\\
\midrule

$\mathrm{DSWRF}_{\mathrm{TOA}}$
    & Downward shortwave radiative flux at TOA
    & W\,m$^{-2}$ & Mean \\

$T_s$
    & Skin temperature over open ocean (SST)
    & K & Snapshot \\

$z_s$
    & Surface topographic height
    & m & Invariant \\

$f_l$
    & Land fraction
    & -- & Invariant \\

$f_o$
    & Ocean fraction
    & -- & Snapshot \\

$f_{\mathrm{si}}$
    & Sea-ice fraction
    & -- & Snapshot \\

$\mathrm{CO}_2$
    & Global-mean atmospheric CO$_2$
    & ppm & Snapshot \\

\midrule
\multicolumn{4}{@{}l}{\textit{Diagnostic variables (output only)}}\\
\midrule

$\mathrm{USWRF}_{\mathrm{TOA}}$
    & Upward shortwave radiative flux at TOA
    & W\,m$^{-2}$ & Mean \\

$\mathrm{ULWRF}_{\mathrm{TOA}}$
    & Upward longwave radiative flux at TOA
    & W\,m$^{-2}$ & Mean \\

$\mathrm{USWRF}_{\mathrm{sfc}}$
    & Upward shortwave radiative flux at surface
    & W\,m$^{-2}$ & Mean \\

$\mathrm{ULWRF}_{\mathrm{sfc}}$
    & Upward longwave radiative flux at surface
    & W\,m$^{-2}$ & Mean \\

$\mathrm{DSWRF}_{\mathrm{sfc}}$
    & Downward shortwave radiative flux at surface
    & W\,m$^{-2}$ & Mean \\

$\mathrm{DLWRF}_{\mathrm{sfc}}$
    & Downward longwave radiative flux at surface
    & W\,m$^{-2}$ & Mean \\

$P$
    & Surface precipitation rate, all phases
    & kg\,m$^{-2}$\,s$^{-1}$ & Mean \\

$(\partial \mathrm{TWP}/\partial t)_{\mathrm{adv}}$
    & Total-water-path tendency from advection
    & kg\,m$^{-2}$\,s$^{-1}$ & Mean \\

$\mathrm{LHF}$
    & Surface latent heat flux
    & W\,m$^{-2}$ & Mean \\

$\mathrm{SHF}$
    & Surface sensible heat flux
    & W\,m$^{-2}$ & Mean \\

$Z_{500}$
    & 500-hPa geopotential height
    & m & Snapshot \\

$T_{850}$
    & 850-hPa air temperature
    & K & Snapshot \\

\bottomrule
\end{tabular} \vspace{8mm}
\end{table*}

\clearpage

\section{Computational Details}
\label{app:compute}

\paragraph{\emph{ReRoute} fine-tuning.}
\emph{ReRoute} is initialized from the released pretrained ACE2 checkpoints and does not retrain the underlying climate emulator from scratch. Fine-tuning for both the ACE2S-SHiELD+$_{\mathrm{no\text{-}RC}}$ and ACE2-ERA5 experiments was performed on two NVIDIA A100 GPUs. Each run completed in
less than 24 hours of wall-clock time and was continued until the training and validation metrics saturated. Experiment-specific optimization hyperparameters are reported in~\cref{app:experiment_details}.

\paragraph{Seed sensitivity and checkpoint selection.}
We repeated fine-tuning with three random seeds and observed no material differences in convergence or downstream performance. We therefore selected a single high-performing checkpoint for the full long-horizon analysis and
used it unchanged across the corresponding factual and counterfactual evaluations. The checkpoint evaluation criterion follows the those for the baseline models, as reported in~\citep{watt2025ace2,clark2026disentangling}.

We report uncertainty across independent stochastic realizations and initial conditions, as specified in \cref{app:experiment_details}. For example, the abrupt-$4\times$CO$_2$ experiment reports variability across its 36-member initial-condition ensemble. We do not additionally report error bars across fine-tuning seeds because variation across the three adaptation runs was small relative to the predictive variability of the stochastic ACE2S emulator.

\paragraph{Compute comparison.}
The compute reported above refers only to \emph{ReRoute} adaptation and excludes the cost of training the released ACE2 backbones and generating the pre-existing ERA5 and SHiELD datasets and evaluation simulations.

Our reported $18\times$ compute reduction applies specifically to the controlled SHiELD benchmark, where a directly comparable
ACE2S-SHiELD+ model trained with additional random-CO$_2$ simulations is available. We compare the GPU-hours required to adapt the released no-RC checkpoint with \emph{ReRoute} against the GPU-hours required to train that controlled-data counterpart. We do not make this compute comparison for ACE2-ERA5, for which no analogous model trained on counterfactual
reanalysis trajectories exists.

\hfill 

\section{Climate-Specific Methodology}
\label{app:climate_details}
This section provides the climate-specific methodological details used in our ACE2 experiments. We first describe the construction of the RRTMG-based CO$_2$ pathway and then the tropical-cyclone tracking procedure used for evaluation.
\subsection{RRTMG CO$_2$ Pathway Construction}
\label{app:rrtmg_pathway}
\looseness=-1 \emph{ReRoute} requires only the first-level mechanistic response of the targeted variable. We compute the CO$_2$ radiative pathway offline using the RRTMG
radiative-transfer scheme \citep{iacono2008radiative}, through its implementation in the Climate Modelling and Diagnostics Toolkit (\texttt{CliMT}) \citep{monteiro2018sympl}.

We first compute a climatological-mean atmospheric state from the training data. Holding this atmospheric state fixed, we run RRTMG at a reference CO$_2$ concentration $\sb_0$ and at perturbed concentrations $c$. This isolates
the instantaneous radiative response to changing CO$_2$, without including the subsequent dynamical adjustment of the atmosphere or ocean.

For each atmospheric level $k$, we compute the change in radiative heating relative to the reference state and normalize by the logarithmic CO$_2$
perturbation:
\begin{equation}
    q_k^{\log}
    :=
    \frac{
        q_k(\sb)-q_k(\sb_0)
    }{
        \log(\sb/\sb_0)
    },
    \label{eq:qlog}
\end{equation}
where $q_k(\sb)$ denotes the RRTMG radiative heating response at CO$_2$ concentration $c$. We similarly compute logarithmic response coefficients for
the longwave radiative fluxes at the top of atmosphere and surface. These response coefficients are computed once and stored as fixed spatial fields.

For a requested CO$_2$ concentration $\sb_t$, the pathway is obtained by scaling
the stored coefficients by
\[
    \log\!\left(\sb_t/\sb_0 \right).
\]
For atmospheric temperature, the corresponding six-hour increment at level
$k$ is
\begin{equation}
    \Delta T_k
    =
    \frac{g\,\Delta t}{c_p\,\Delta p_k}
    q_k^{\log}
    \log\!\left(\frac{\sb_t}{\sb_0}\right),
    \label{eq:co2-temp-pathway}
\end{equation}
where $\Delta p_k$ is the pressure thickness of the layer, $c_p$ is the
specific heat capacity of air, and $\Delta t=6$\,h is the ACE2-SOM timestep.
The stored longwave-flux response coefficients are scaled in the same way for
$\mathrm{ULWRF}_{\mathrm{toa}}$,
$\mathrm{DLWRF}_{\mathrm{sfc}}$, and
$\mathrm{ULWRF}_{\mathrm{sfc}}$.

This defines the mechanistic pathway
\[
    \Gamma(\sb_0,\sb_t),
\]
which captures only the instantaneous first-level radiative effect of the CO$_2$ intervention. \textit{Notably, it does not prescribe the full downstream climate response}. Subsequent changes in temperature, moisture, circulation, surface fluxes, and the slab ocean emerge through the autoregressive ACE2-SOM dynamics. In all experiments, the stored pathway coefficients are kept fixed.

\subsection{Tropical Cyclone Tracking Details}
\label{app:tc-tracking}

\looseness=-1
As an additional evaluation of the weather statistics generated by ACE2-ERA5, we compare tropical-cyclone (TC) climatology between ERA5, the pretrained ACE2-ERA5 baseline, and ACE2-ERA5+\emph{ReRoute} over the held-out 2001--2010 period using six-hourly $1^\circ$ fields. TCs are detected and tracked with exactly the same \textbf{TempestExtremes v2.1} pipeline~\citep{ullrich2021tempestextremes} for all three datasets, so differences in the resulting tracks arise from the generated atmospheric fields rather than from the detection procedure. Candidate storm centers are first identified with \texttt{DetectNodes} as local minima in mean sea-level pressure satisfying a closed-contour pressure drop of at least $200$~Pa within a $5.5^\circ$ radius. Following the ACE2-ERA5 evaluation protocol~\citep{watt2025ace2}, the warm-core criterion is expressed using the ACE2 $T_3$ atmospheric layer, whose pressure range is approximately 250--400~hPa, because ACE2 does not directly predict the upper-level geopotential thickness conventionally used by the tracker. A candidate must exhibit a decrease of at least $0.4$~K in $T_3$ away from the storm center, approximately matching the hydrostatic thickness criterion used in the standard TempestExtremes configuration. Detected centers are then linked through time using \texttt{StitchNodes} to construct complete cyclone trajectories. We apply identical detection, stitching, duration, and displacement criteria to ERA5, Baseline, and Baseline+\emph{ReRoute}, and summarize the resulting spatial track climatology over the full 10-year period. This controlled use of a common tracker permits a direct comparison of the model-generated TC statistics against ERA5.

\section{Climate Experiment Details and Additional Results}
\label{app:experiment_details}

This section provides the full experimental configurations and additional results for our climate evaluations. We first describe the controlled coupled-climate experiments with ACE2S-SHiELD+, followed by the historical reanalysis experiments with ACE2-ERA5.

\subsection{Coupled Climate Simulation}
\label{app:ace2s_shield}

\paragraph{Experimental setup.}
We evaluate \emph{ReRoute} on the released ACE2S-SHiELD+$_{\mathrm{no-RC}}$ checkpoint~\citep{clark2026disentangling}, a six-hour autoregressive global climate emulator with the same stochastic SFNO~\citep{bonev2023spherical} architecture and prognostic variables as ACE2S-SHiELD+, but trained without the random-CO$_2$ simulations used to explicitly decorrelate SST and CO$_2$. The no-RC checkpoint is trained on 90 years of data: 45 years of AMIP simulations and 45 years of equilibrium-climate simulations spanning $1\times$, $2\times$, and $4\times$CO$_2$. Validation uses held-out AMIP and equilibrium-climate years, and checkpoint selection uses 120 years of inline inference across the three equilibrium climates (details about the dataset see \cref{app:data_shield}). Although the equilibrium reference simulations originate from ocean-coupled SHiELD, training, validation, and inline inference are performed with prescribed SST, following~\citet{clark2026disentangling}. Thus, the model observes multiple equilibrium CO$_2$ climates but receives no training trajectories in which CO$_2$ is varied independently of the corresponding climate state. We use the released no-RC checkpoint as the pretrained baseline throughout this experiment.

\hfill

\paragraph{Counterfactual query.}
We evaluate an abrupt-$4\times$CO$_2$ intervention using the SOM-coupled SHiELD benchmark of~\citet{clark2026disentangling}. Each simulation is initialized from the $1\times$CO$_2$ climate and the atmospheric CO$_2$ concentration is instantaneously increased to $4\times$CO$_2$, while the slab ocean subsequently evolves interactively through the surface fluxes predicted by the atmospheric model. This intervention places the model in a forcing--state combination absent from its training distribution: although equilibrium climates at both $1\times$ and $4\times$CO$_2$ are represented during training, the model never observes a $1\times$ climate state subjected instantaneously to $4\times$CO$_2$ forcing. We follow the reference protocol and evaluate a 36-member ensemble of 90-day rollouts, with initial conditions spaced one month apart over three years to average weather variability and sample the seasonal cycle. The corresponding SHiELD simulations provide an independent reference trajectory for every ensemble member, enabling direct quantitative evaluation of the intervention response.

\paragraph{\emph{ReRoute} configuration.}
For this experiment, the CO$_2$ pathway is the fixed RRTMG radiative-response pathway described in \cref{app:rrtmg_pathway}, with the reference concentration set to $\sb_0=363.43\,\mathrm{ppm}$ and pathway scaling fixed to one. The pretrained backbone is always conditioned on $\sb_0$, while the abrupt change in CO$_2$ is supplied through $\Gamma(\sb_0,\sb_t)$. The baseline instead receives the requested CO$_2$ concentration through its native conditioning. The coupled evaluation is performed without random-CO$_2$ training trajectories or abrupt-$4\times$CO$_2$ supervision.

\paragraph{Evaluation Metrics}
Following~\citet{clark2026disentangling}, our primary long-horizon metric is the channel-mean area-weighted RMSE of the normalized time- and ensemble-mean bias.

Let $\tilde y_{c,m,s}(\phi,\lambda)$ and $\tilde{\hat y}_{c,m,s}(\phi,\lambda)$ denote the normalized SHiELD reference and model prediction for channel $c$, ensemble member $m$, and six-hourly timestep $s$. Here, $\phi$ and $\lambda$ denote latitude and longitude, respectively. We first form the time-mean bias map for each member,
\[
\Delta_{c,m}(\phi,\lambda)
=
\frac{1}{T}
\sum_{s=1}^{T}
\left[
\tilde{\hat y}_{c,m,s}(\phi,\lambda)
-
\tilde y_{c,m,s}(\phi,\lambda)
\right],
\]
and then average the bias maps across the $M=36$ initial-condition members,
\[
\overline{\Delta}_{c}(\phi,\lambda)
=
\frac{1}{M}\sum_{m=1}^{M}
\Delta_{c,m}(\phi,\lambda).
\]
For each output channel, the spatial error is
\[
\mathrm{RMSE}_{c}
=
\left(
\sum_{\phi,\lambda}
w_{\phi,\lambda}
\overline{\Delta}_{c}(\phi,\lambda)^2
\right)^{1/2},
\]
where $w_{\phi,\lambda}$ are normalized spherical-area weights. We then average across the $C=51$ ACE2S-SHiELD+ output channels,
\[
\alpha_{\mathrm{ens}}
=
\frac{1}{C}
\sum_{c=1}^{C}
\mathrm{RMSE}_{c}.
\]
Averaging over time and ensemble members before the nonlinear spatial RMS matches the evaluation convention of~\citet{clark2026disentangling}. 

\paragraph{Member-wise uncertainty.}
To quantify variation across the 36 independent initial conditions, we additionally compute the same aggregate error separately for each ensemble member,
\[
\alpha_m
=
\frac{1}{C}\sum_{c=1}^{C}
\left(
\sum_{\phi,\lambda}
w_{\phi,\lambda}
\Delta_{c,m}(\phi,\lambda)^2
\right)^{1/2},
\]
and report the mean and population standard deviation of $\{\alpha_m\}_{m=1}^{M}$. This is the statistic shown with error bars in \cref{fig:shift-severity}. The native no-RC model obtains $0.1619\pm0.0116$, while \emph{ReRoute} obtains $0.1324\pm0.0101$, corresponding to an $18.2\%$ reduction in mean member-wise error. The ensemble-mean-first metric $\alpha_{\mathrm{ens}}$ and the member-wise statistic quantify complementary properties: the former measures error after weather noise has been averaged across the ensemble, while the latter exposes variability across independent initial conditions.

\hfill

\subsection{Historical Reanalysis with ACE2-ERA5}
\label{app:ace2_era5}

\paragraph{Experimental setup.}
We evaluate \emph{ReRoute} on ACE2-ERA5~\citep{watt2025ace2}, a six-hour autoregressive global atmospheric emulator trained directly on ERA5 reanalysis. ACE2-ERA5 operates at $1^\circ$ horizontal resolution with eight atmospheric layers and is conditioned on prescribed sea-surface temperature (SST), sea ice, incoming solar radiation, global-mean atmospheric CO$_2$, and static surface fields. The released model is trained on ERA5 data from 1940--1995, 2011--2019, and 2021--2022, with 1996--2000 used for validation and 2001--2010 and 2020 reserved for testing (details about the dataset are provided in \cref{app:data_era5}). Because SST, sea ice, and atmospheric CO$_2$ co-evolve throughout the historical record, ACE2-ERA5 never observes controlled variations that independently separate their effects. We use the released ACE2-ERA5 checkpoint as the pretrained baseline and initialize \emph{ReRoute} from exactly the same model. The architecture, normalization statistics, conservation corrections, prognostic variables, prescribed boundary conditions, and six-hour autoregressive timestep are unchanged; \emph{ReRoute} modifies only the pathway through which CO$_2$ enters the emulator.\\

\paragraph{Counterfactual query.}
We ask how the historical atmosphere would evolve if atmospheric CO$_2$ had remained fixed at its 1940 value while the remaining observed boundary conditions followed their historical trajectories. We therefore perform two otherwise identical 81-year rollouts from January 1940 through December 2020.

In the factual rollout, CO$_2$ follows its historical trajectory; in the intervention rollout, CO$_2$ is held fixed at its 1940 value, while SST, sea ice, incoming solar radiation, and all other prescribed forcings remain identical to the historical experiment. This intervention breaks the historical SST--CO$_2$ association present in the ERA5 training distribution. No observed atmosphere exists under the same realized 1940--2020 SST and sea-ice trajectory with CO$_2$ fixed, so this experiment is evaluated through the physical structure of the predicted response. 

\hfill  

\paragraph{\emph{ReRoute} adaptation.}
We initialize \emph{ReRoute} from the released ACE2-ERA5 checkpoint and adapt it using only factual historical ERA5 samples. 

Training uses two-step autoregressive rollouts with an MSE objective applied at both rollout steps. We optimize with AdamW using learning rate $2\times10^{-6}$ and zero weight decay, with a linear warmup from $0.1$ to $1$ times the base learning rate over the first 60 optimization steps. We use batch size four, no gradient accumulation or mixed precision, and exponential-moving-average model parameters for validation. The RRTMG pathway is constructed as described in \cref{app:rrtmg_pathway}, with reference concentration $\sb_0$ set to the mean CO$_2$ concentration of the ACE2-ERA5 training data.. 

\hfill  

\paragraph{Fixed-CO$_2$ response.}
We evaluate the intervention using the spatial response of near-surface air temperature. Let $T^{\mathrm{hist}}_{\mathrm{2m}}(\phi,\lambda,t)$ and $T^{\mathrm{fix}}_{\mathrm{2m}}(\phi,\lambda,t)$ denote the predicted 2-m temperature under historical and fixed CO$_2$, respectively. We compute the time-mean response
\begin{equation}
    \Delta T_{\mathrm{2m}}(\phi,\lambda)
    =
    \overline{T}^{\,\mathrm{fix}}_{\mathrm{2m}}(\phi,\lambda)
    -
    \overline{T}^{\,\mathrm{hist}}_{\mathrm{2m}}(\phi,\lambda),
    \label{eq:era5-fixed-response}
\end{equation}
where $\phi$ and $\lambda$ denote latitude and longitude, respectively, and the overbar denotes the mean over the 1940--2020 rollout. Because SST and sea ice are identical between the two experiments, $\Delta T_{\mathrm{2m}}$ isolates the change produced by altering the emulator's CO$_2$ pathway while retaining the same prescribed surface evolution. \\

The native ACE2-ERA5 baseline produces widespread negative $\Delta T_{\mathrm{2m}}$, indicating that fixing CO$_2$ removes a substantial portion of the warming supported by the evolving boundary conditions. \emph{ReRoute} produces a much weaker fixed-minus-historical response over most of the globe, preserving substantially more of the surface warming encoded by the prescribed SST and sea-ice evolution.

\clearpage

\begin{figure}[thbp] \vspace{4mm}
    \centering
    \includegraphics[width=\linewidth]{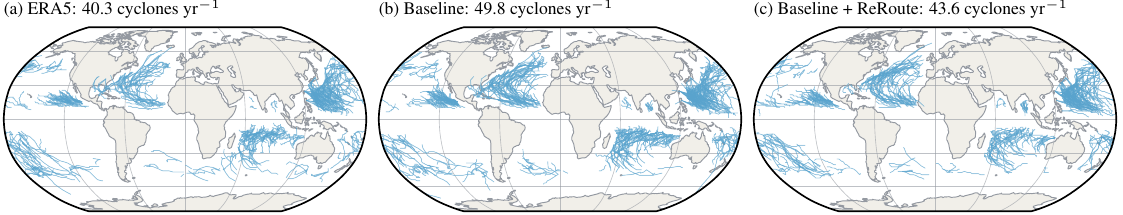}
    \caption{
    \textbf{\emph{ReRoute} improves climatology of tropical cyclones in ACE2-ERA5.}
    Tropical cyclone tracks detected in ERA5, the pretrained ACE2-ERA5 baseline, and ACE2-ERA5 + \emph{ReRoute} over the held-out evaluation period. The same detection and tracking procedure of tropical cyclones are applied to all three datasets. \emph{ReRoute} produces a spatial distribution of cyclone activity that more closely matches the ERA5 reference than the pretrained baseline.
    }
    \label{fig:tc-tracks} \vspace{3mm}
\end{figure}

\paragraph{Tropical-cyclone evaluation.}
We additionally evaluate whether the factual climate generated by \emph{ReRoute} reproduces the tropical-cyclone statistics of ERA5. Following the ACE2 evaluation protocol of~\citet{watt2025ace2}, we analyze tropical-cyclone-like features over the held-out 2001--2010 period using the TempestExtremes tracking package. We apply the same detection and tracking parameters to ERA5, the pretrained ACE2-ERA5 baseline, and ACE2-ERA5+\emph{ReRoute}. The tracker follows the default tropical-cyclone configuration used for ACE2, with the warm-core criterion adapted to the variables predicted by the emulator: instead of upper-level geopotential thickness, we use the $T_3$ atmospheric layer, corresponding approximately to 250--400\,hPa, and require temperature to decrease by $0.4$\,K away from the cyclone center. We compare the resulting track climatology and global occurrence frequency across the three datasets. Because the identical tracker is applied to ERA5 and both emulator rollouts, differences directly reflect the generated atmospheric statistics rather than differences in event identification. \\

As shown in \cref{fig:tc-tracks}, the proposed \emph{ReRoute} approach produces a distribution of tropical-cyclone tracks that more closely matches the ERA5 reference than the pretrained ACE2-ERA5 baseline. Details are summarized in \cref{app:tc-tracking}.

\end{document}